\documentclass[10pt,twocolumn,letterpaper]{article}

\usepackage{preprint}

\usepackage{cuted}               
\usepackage[T1]{fontenc}
\usepackage[utf8]{inputenc}
\usepackage{amsthm,mathtools}
\usepackage{bm,microtype}
\usepackage{tabularx,array,multirow}
\usepackage{colortbl}
\definecolor{oursbg}{RGB}{226,239,249}
\definecolor{refgray}{gray}{0.42}
\newcommand{\ours}{\rowcolor{oursbg}}
\newcommand{\grp}[1]{\textit{\textcolor{refgray}{#1}}}
\usepackage{algorithm,algpseudocode}
\usepackage{tikz}
\usepackage{wrapfig}
\usepackage[normalem]{ulem}
\usepackage{float}
\usepackage{needspace}

\AtBeginDocument{\captionsetup{skip=3pt}\setlength{\belowcaptionskip}{-3pt}}
\setlist[itemize]{leftmargin=1.1em,itemsep=1pt,topsep=3pt}
\setlist[enumerate]{leftmargin=1.4em,itemsep=1pt,topsep=3pt}

\definecolor{linkblue}{rgb}{0.21,0.49,0.74}
\usepackage[pagebackref,breaklinks,colorlinks,allcolors=linkblue]{hyperref}
\hypersetup{pdftitle={FlowMap-OPD: On-Policy Distillation for Few-Step Flow-Map Generators},pdfauthor={Zhiqi Li, Cheng Wan, Bo Zhu}}

\newtheorem{theorem}{Theorem}

\theoremstyle{plain}
\newtheorem{appendixtheorem}{Theorem}[section]
\newtheorem{appendixproposition}[appendixtheorem]{Proposition}
\newtheorem{appendixcorollary}[appendixtheorem]{Corollary}
\theoremstyle{definition}
\newtheorem{appendixassumption}[appendixtheorem]{Assumption}
\theoremstyle{plain}

\theoremstyle{definition}

\theoremstyle{remark}

\newcommand{\E}{\mathbb E}

\newcommand{\KL}{D_{\mathrm{KL}}}
\newcommand{\Tcal}{\mathcal{R}}
\newcommand{\sg}{\operatorname{sg}}
\newcommand{\Id}{\mathrm I}
\newcommand{\norm}[1]{\left\|#1\right\|}

\newcommand{\Gcal}{\mathcal G}
\newcommand{\Lcal}{\mathcal L}

\newcolumntype{Y}{>{\raggedright\arraybackslash}X}
\allowdisplaybreaks[1]

\title{FlowMap-OPD: On-Policy Distillation\\for Few-Step Flow-Map Generators}
\author{Zhiqi Li$^{1}$ \qquad Cheng Wan$^{2}$ \qquad Bo Zhu$^{1}$\\[2pt]
$^{1}$Georgia Institute of Technology \qquad $^{2}$Cornell University\\[2pt]
{\tt\small zli3167@gatech.edu}\\[3pt]
{\small Project: \url{https://zhiqili-cg.github.io/Flowmap_OPD_project/}}\\
{\small Code: \url{https://github.com/ZhiqiLi-CG/Flowmap_OPD_source}}}

\begin{document}
\maketitle
\begin{strip}
\centering
\includegraphics[width=\textwidth]{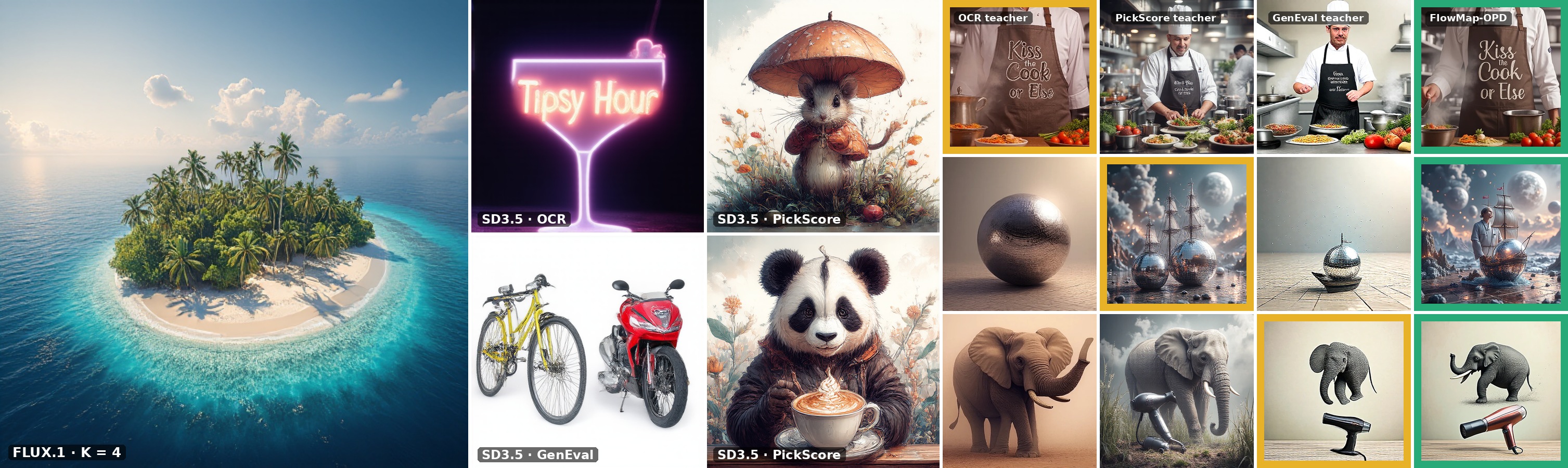}
\captionsetup{hypcap=false}
\captionof{figure}{\textbf{One few-step student distilled from three specialist teachers.}
Left, middle: FlowMap-OPD student samples on FLUX.1 and SD3.5; tags give the task of each prompt. Right (FLUX.1): an OCR, a PickScore, and a GenEval prompt (rows) for the three teachers and the student (columns). Gold borders mark the specialist trained for the row's task; green borders mark the student, which handles all three.}
\captionsetup{hypcap=true}
\label{fig:teaser}
\end{strip}
\begin{abstract}
\upshape Few-step flow-map generators, such as MeanFlow and consistency models, sample in a few long-range jumps. On-policy distillation (OPD) trains a student on states drawn from its own sampler, but existing methods compare teacher and student through the local stochastic transitions of a multi-step sampler, which a flow-map student never deploys. We present FlowMap-OPD, the first systematic study of OPD for flow-map students. It separates state collection from supervision: the student collects states with its native few-step sampler, while the teacher supervises the flow map, velocities induced from it, or the instantaneous velocity linked to the map by a separately weighted consistency penalty. We find that instantaneous-velocity supervision with a small consistency weight, our chosen form, transfers the teacher's gains best, whereas supervising the deployed map, although the natural choice, transfers them worst when teacher and student start from different pretrained models. This gives a counter-intuitive rule: \emph{sample with the flow map, learn the velocity}. As a further benefit, this form applies to any teacher that defines a velocity, including multi-step ones, and to students initialized from multi-step models. On FLUX and SD3.5, a single few-step student consolidates three specialist teachers within 300 updates, matches or exceeds each on its own benchmark (e.g., OCR 0.883 versus 0.850 on FLUX), and outperforms both multi-reward reinforcement learning and LoRA merging.
\end{abstract}

\begin{figure*}[t]
\centering
\begin{minipage}[t]{0.33\textwidth}
\centering
\includegraphics[width=\linewidth]{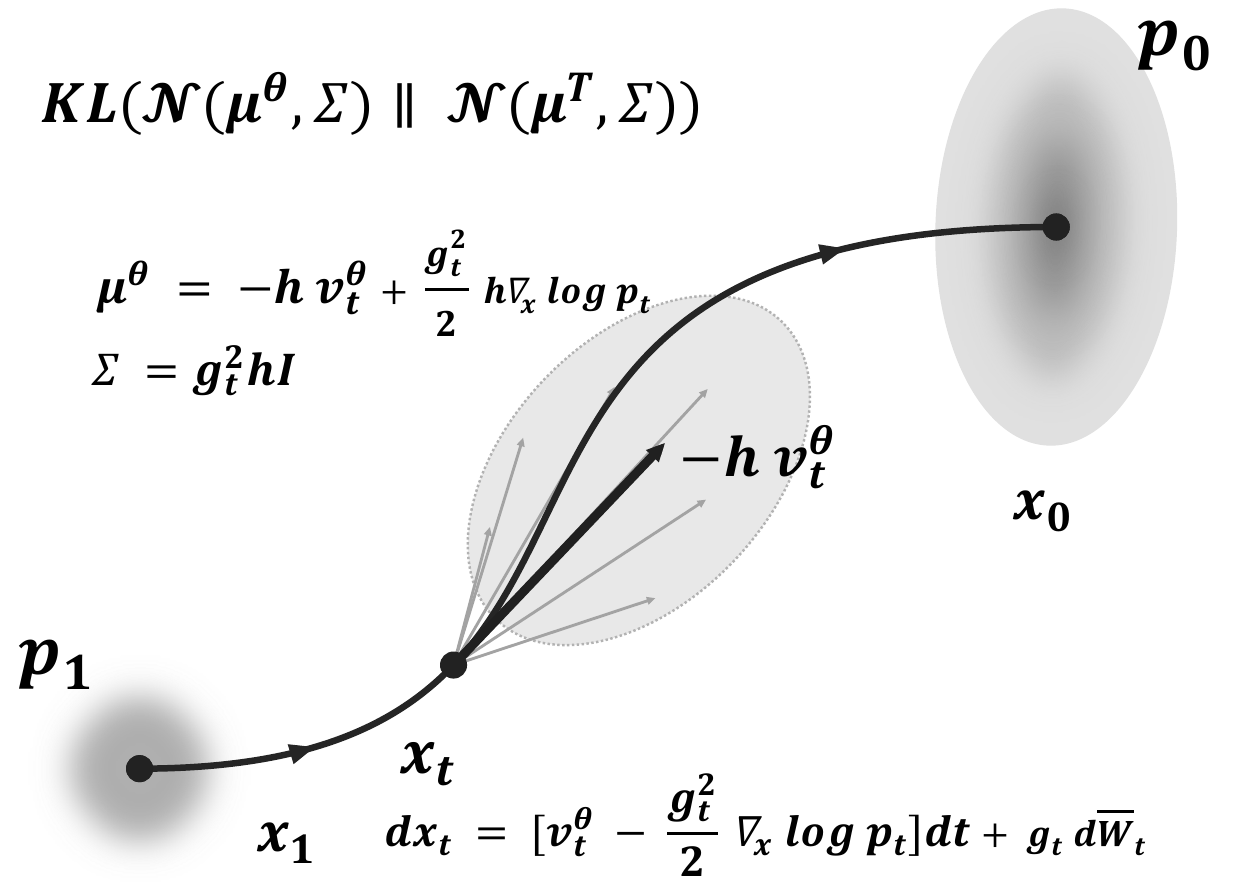}\par
\vspace{2pt}
{\small\bfseries (1) Flow/Diffusion OPD\par}
\end{minipage}\hfill
\begin{minipage}[t]{0.33\textwidth}
\centering
\includegraphics[width=\linewidth]{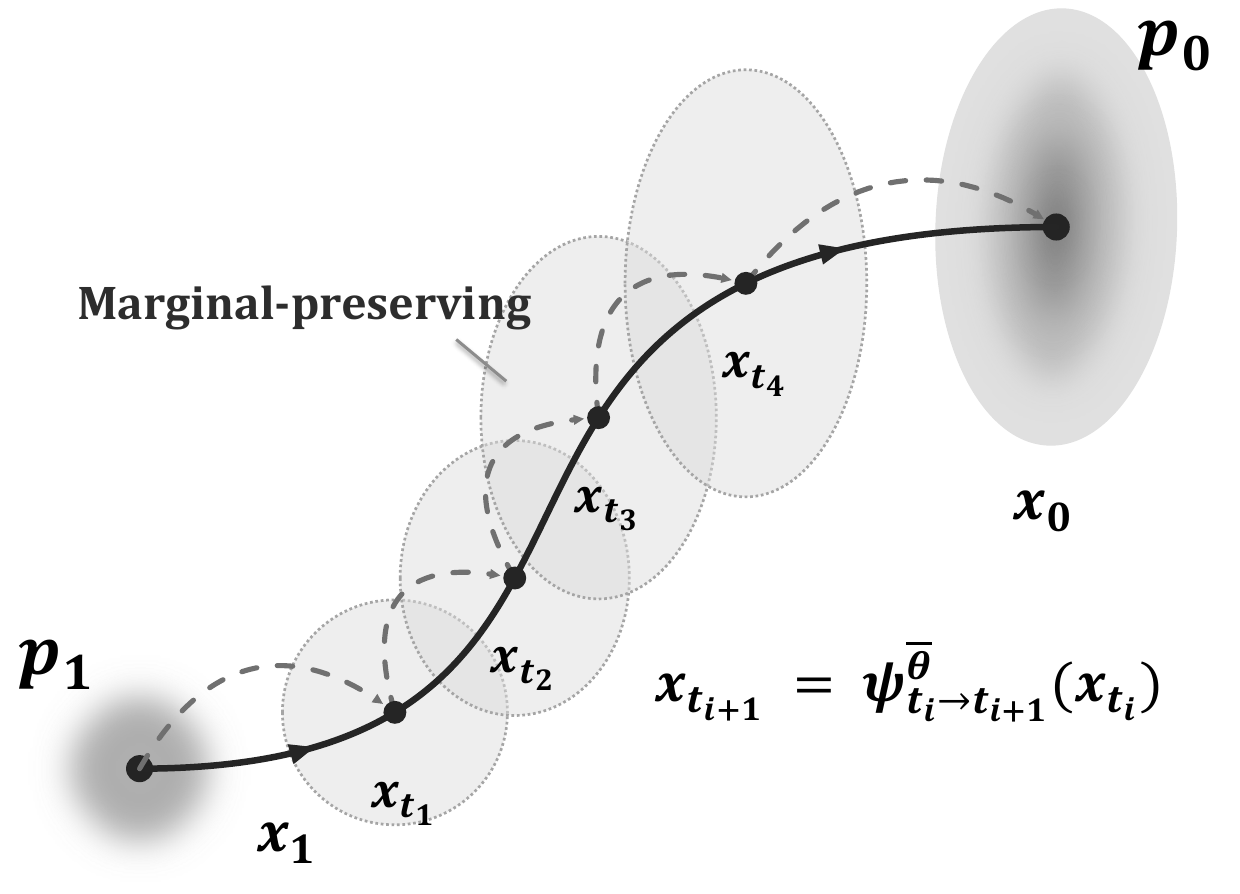}\par
\vspace{2pt}
{\small\bfseries (2) FlowMap-OPD rollout\par}
\end{minipage}\hfill
\begin{minipage}[t]{0.33\textwidth}
\centering
\includegraphics[width=\linewidth]{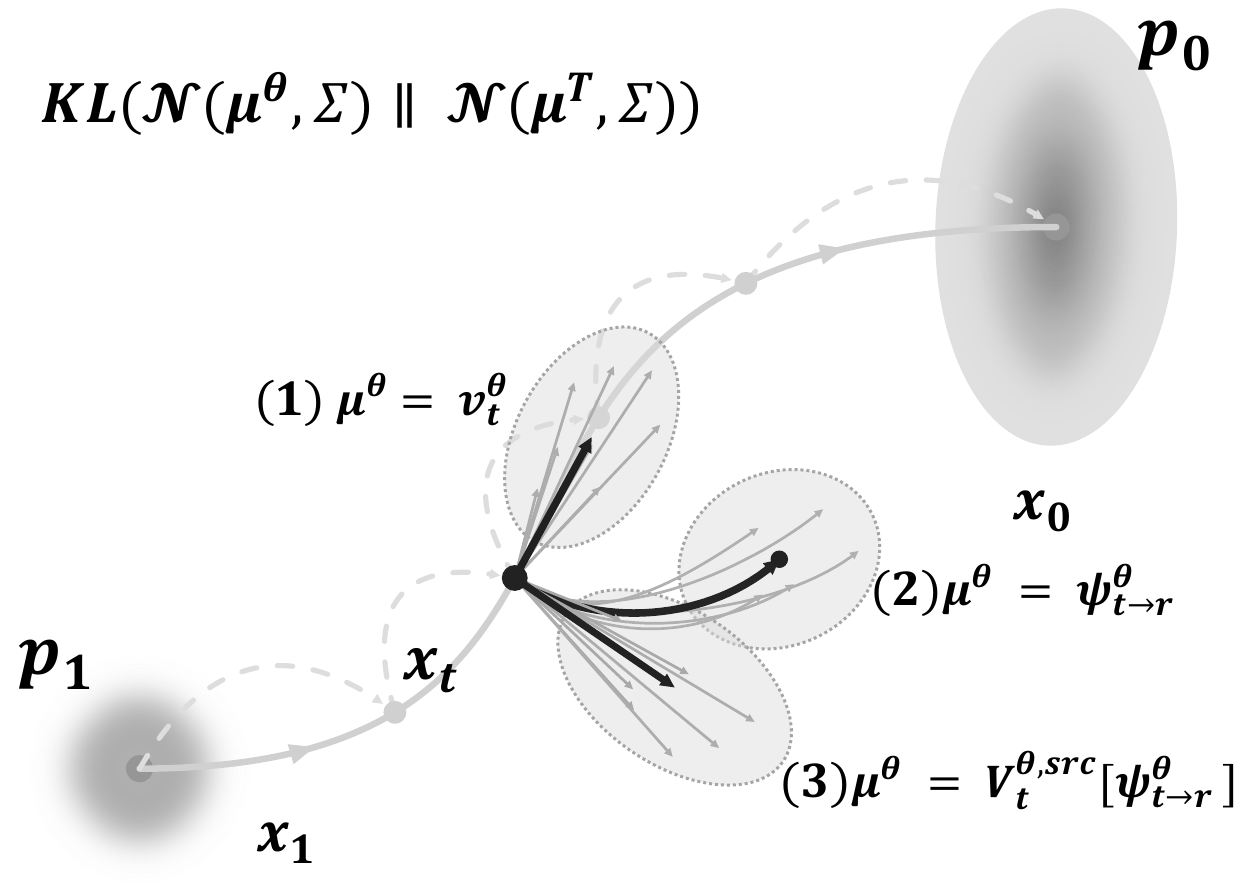}\par
\vspace{2pt}
{\small\bfseries (3) FlowMap-OPD supervision\par}
\end{minipage}
\caption{\textbf{FlowMap-OPD separates the rollout from the supervised quantity.}
\textbf{(1)} Flow/Diffusion OPD~\citep{fang2026flowopd,li2026diffusionopd} uses one local transition both for the rollout and for the teacher--student comparison.
\textbf{(2)} FlowMap-OPD collects states with the student's native long-range map, $x_{t_{i+1}}=\psi^{\theta}_{t_i\to t_{i+1}}(x_{t_i})$; the bar on $\theta$ marks that no gradient flows through the rollout.
\textbf{(3)} At those states, a separately chosen Gaussian kernel compares student and teacher; its mean is built from the instantaneous velocity (ours), the flow map, or a map-induced velocity.}
\label{fig:separation}
\end{figure*}

\section{Introduction}\label{sec:intro}

Reward post-training turns a few-step generator into a specialist: one model renders text, another follows compositional prompts, and a third matches human preference (Figure~\ref{fig:teaser}). Consolidating such specialists into one few-step model is a distillation problem, for which on-policy distillation (OPD) is the natural tool: it supervises a student at the states visited during its own generation, which removes the mismatch between its training and inference distributions~\citep{agarwal2023opd,gu2024minillm}. DiffusionOPD and Flow-OPD bring this principle to diffusion and flow-matching models by comparing teacher and student transition distributions at every sampling step~\citep{li2026diffusionopd,fang2026flowopd}. Both give one student transition kernel two jobs: it generates the rollout, and it is compared with the teacher.

A flow-map student has no such kernel. Few-step flow-map generators such as MeanFlow and consistency models~\citep{geng2025meanflow,song2023consistency,lu2025simplifying} sample with a few long-range jumps, each deterministic and without a local Gaussian transition to compare. Trajectory-level OPD would thus require a many-step stochastic sampler that the student never deploys, with four times as many network evaluations as native few-step sampling in our setting (Section~\ref{sec:separation-exp}). The choice of what to compare is open as well, since a flow-map network provides several representations of transport~\citep{geng2025meanflow,boffi2024fmm}: the long-range map it samples with, the instantaneous velocity (the map's derivative at $r=t$), and velocities induced from the map through its consistency identities~\citep{huang2026meanflownft}. Since a teacher could supervise any of them, we ask two questions: \emph{must a flow-map student use the same representation for rollout and teacher comparison, and if not, which representation should the teacher supervise?}

The answer to the first question is no. The OPD objective is a sum of local comparisons, so its expectation depends on the rollout only through the state marginals at the comparison times. The \emph{rollout} can therefore collect states with the student's native few-step sampler, while a separately chosen kernel compares student and teacher there. We call this \emph{rollout--kernel separation} (Figure~\ref{fig:separation}) and build \emph{FlowMap-OPD} on it.

The answer to the second question is not the obvious one. With the supervised quantity now a design choice, we study how the choices \emph{transfer} the reward-specific capabilities of post-trained teachers to flow-map students. The natural choice is the flow map itself, since it is what the student deploys. Surprisingly, it transfers worst when teacher and student come from different pretrained models, because the target then also carries the difference between the two pretrained maps (Section~\ref{sec:imagenet}). In contrast, matching the teacher's instantaneous velocity, linked to the map by a separately weighted consistency penalty, has three advantages. It transfers best, whether or not teacher and student share a base model (Sections~\ref{sec:imagenet},~\ref{sec:t2i}), and it applies both to any teacher with a velocity, including multi-step ones, and to students initialized from multi-step models (Section~\ref{sec:separation-exp}). Here, the penalty propagates this supervision of the instantaneous velocity to the long-range map the student samples with (Section~\ref{sec:decouple}). This gives a counter-intuitive rule: \emph{sample with the flow map, learn the velocity}.

To our knowledge, this is the first systematic study of on-policy distillation for flow-map generators, with the following contributions:
\begin{itemize}
\item \textbf{Rollout--kernel separation} (Section~\ref{sec:rk-theory}). A flow-map student trains on its native few-step rollouts, and the supervised quantity is chosen separately.
\item \textbf{Decoupled velocity supervision} (Sections~\ref{sec:decouple},~\ref{sec:imagenet}). Velocity and consistency errors are weighted separately; a bound connects both to the deployed map, and this form transfers best (ImageNet FID 22.68 to 18.23; Table~\ref{tab:imagenet}).
\item \textbf{Specialist consolidation} (Sections~\ref{sec:t2i},~\ref{sec:separation-exp}). In 300 updates, one few-step student matches or exceeds three specialists on FLUX and SD3.5 (e.g., OCR 0.883 against 0.850; Figures~\ref{fig:teaser},~\ref{fig:capability-comparison}), where multi-reward reinforcement learning and LoRA merging stay below every specialist. A rollout needs five network evaluations or fewer, and the method also covers consistency-model students.
\end{itemize}

\begin{figure*}[t]
\centering
\includegraphics[width=\textwidth]{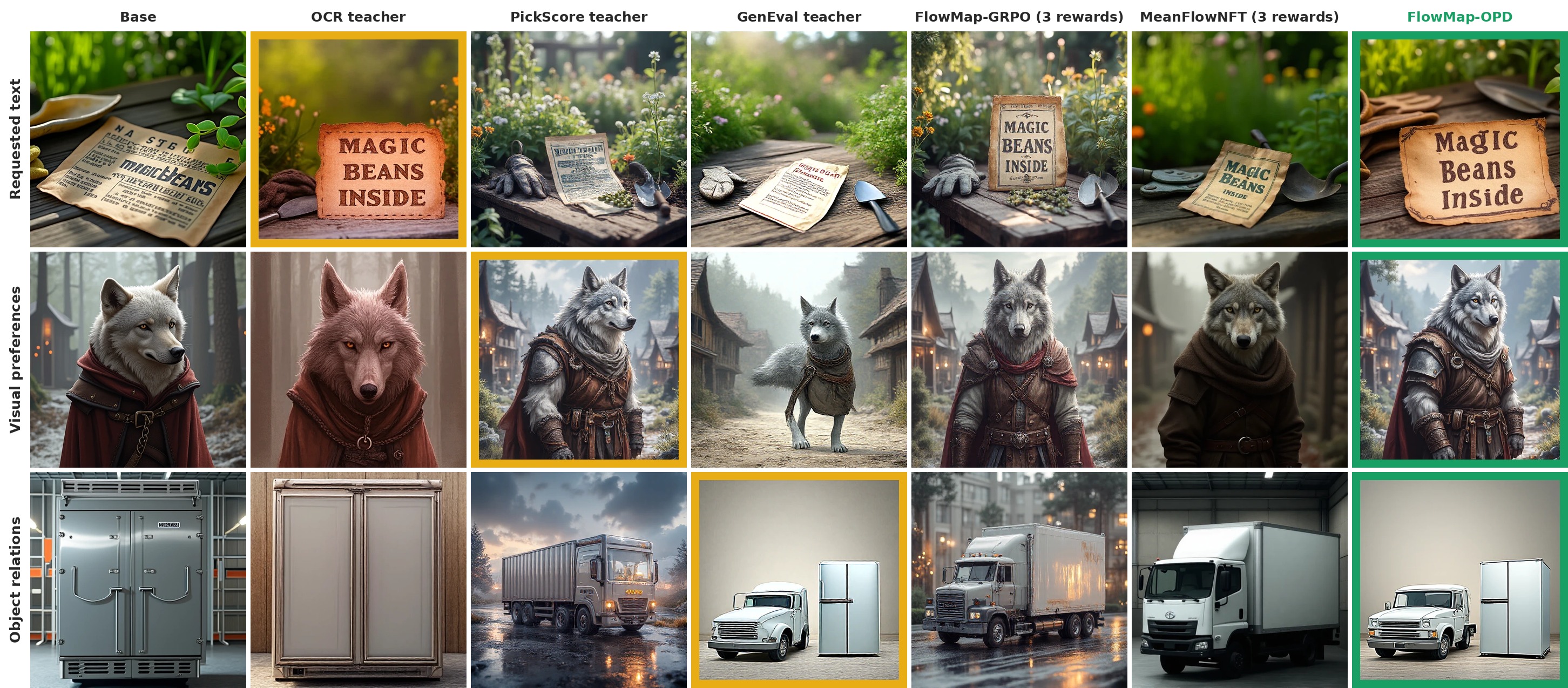}
\caption{\textbf{Specialist capabilities in one few-step student.} Each row uses the same prompt, initial noise, and sampler ($K=4$) for the base, the three specialists, FlowMap-GRPO and MeanFlowNFT trained directly on the three rewards, and the FlowMap-OPD student. Gold borders mark the specialist trained for the row's task; green borders mark the FlowMap-OPD student.}
\label{fig:capability-comparison}
\end{figure*}

\section{Background}\label{sec:background}
We write $\theta$ for the student and $T$ for the fixed teacher. Generation runs from noise $x_1\sim p_1$ at $t=1$ to data at $t=0$ along a probability path $\{p_t\}$ with $p_0=p_{\mathrm{data}}$~\citep{lipman2023flowmatching,song2021scorebased}.

\paragraph{Flow maps and consistency.}\label{sec:fm-consistency}\label{sec:fm-predictions}
For generation, a velocity field $v$ transports noise to data by $dx_t/dt=v(x_t,t)$. Its flow map $\psi_{t\to r}$ sends a state at time $t$ to time $r\leq t$ along this ODE, so that $\psi_{t\to t}=\mathrm{Id}$, $\psi_{s\to r}\circ\psi_{t\to s}=\psi_{t\to r}$, and $p_r=(\psi_{t\to r})_\#p_t$. A flow map and its velocity satisfy the source and destination identities~\citep{geng2025meanflow,boffi2024fmm}
\begin{equation}
 \begin{aligned}
 \partial_t\psi_{t\to r}(x)+(J_x\psi_{t\to r}(x))v(x,t)&=0,\\
 \partial_r\psi_{t\to r}(x)-v(\psi_{t\to r}(x),r)&=0,
 \end{aligned}
 \label{eq:fm-consistency-relations}
\end{equation}
and, conversely, a family of maps with $\psi_{t\to t}=\mathrm{Id}$ that satisfies either identity with $v$ is the flow map of $v$ (Supp.~\ref{app:exact-transfer}). For a learned pair we measure the violation by the residuals
\begin{equation}
 \begin{aligned}
 c^{\mathrm{src}}(x,t,r)&=-\partial_t\psi_{t\to r}(x)-(J_x\psi_{t\to r}(x))v(x,t),\\
 c^{\mathrm{dst}}(x,t,r)&=\partial_r\psi_{t\to r}(x)-v(\psi_{t\to r}(x),r).
 \end{aligned}
 \label{eq:fm-csrc}
\end{equation}
For a model with parameters $\theta$ we write $c^{\theta,\mathrm{src}}$ and $c^{\theta,\mathrm{dst}}$.

\paragraph{Few-step flow-map generators.}\label{sec:instances}
MeanFlow~\citep{geng2025meanflow} parameterizes the map by an average velocity, $\psi^\theta_{t\to r}(x)=x-h\,v^\theta_{t\to r}(x)$ with $h=t-r$, whose diagonal is the instantaneous velocity, $v^\theta(x,t)=v^\theta_{t\to t}(x)$. Its native sampler composes long-range maps on a coarse grid $\Gcal=(1=t_0>\cdots>t_N=0)$:
\begin{equation}
 x_{t_{i+1}}=\psi^\theta_{t_i\to t_{i+1}}(x_{t_i}),\qquad x_{t_0}\sim p_1.
 \label{eq:meanflowstep}
\end{equation}
Following FlowMap-GRPO~\citep{li2026flowmapgrpo}, we index a sampler by $K$ ($K+1$ network evaluations) and use $K$ throughout. Consistency models~\citep{song2023consistency,lu2025simplifying} instead predict a fixed endpoint $\psi^\theta_{t\to0}$ and re-noise between steps (Supp.~\ref{app:endpoint-kernels}).

\paragraph{Trajectory-level OPD.}\label{sec:rk-trajectory}
DiffusionOPD and Flow-OPD~\citep{li2026diffusionopd,fang2026flowopd} sample with the marginal-preserving SDE $dx_t=b_t(x_t)\,dt+g_t\,d\overline W_t$, $b_t(x)=v(x,t)-\tfrac12g_t^2\nabla_x\log p_t(x)$, whose discretization on $\Gcal$ gives Gaussian transition kernels $P_i$. The student samples a trajectory $\tau$ with its own kernel, $x_{t_{i+1}}\sim P_i^\theta(\cdot\mid x_{t_i},\Gcal)$ from $x_{t_0}\sim p_1$, and is compared with the teacher at every visited state:
\begin{equation}
 \begin{aligned}
 \mathcal J_P^\Gcal(\theta)&=\mathbb E_{\tau}\Bigl[\sum_{i=0}^{N-1}\ell_i^P(\theta,T;x_{t_i},\Gcal)\Bigr],\\
 \ell_i^P&=\KL\bigl(P_i^\theta(\cdot\mid x_{t_i},\Gcal)\,\big\|\,P_i^T(\cdot\mid x_{t_i},\Gcal)\bigr).
 \end{aligned}
 \label{eq:rk-trajectory-objective}
\end{equation}
The kernels are Gaussian, so the KL is a weighted squared difference of drifts. The same kernel $P^\theta$ both generates the states and is compared with the teacher (Figure~\ref{fig:separation}(1)).

\begin{table*}[t]
\centering\small

\setlength{\tabcolsep}{4pt}\renewcommand{\arraystretch}{0.94}
\begin{tabularx}{\textwidth}{@{}l*{8}{>{\centering\arraybackslash}X}@{}}
\toprule
& \multicolumn{3}{c}{\textbf{Task scores} $\uparrow$} & \multicolumn{5}{c}{\textbf{Quality preservation (DrawBench)} $\uparrow$}\\
\cmidrule(lr){2-4}\cmidrule(l){5-9}
Model / supervision & GenEval & OCR & PickScore & PickScore & Aesthetic & DeQA & ImgRwd & UniRwd\\
\midrule
Base & 0.504 & 0.349 & 20.98 & 21.63 & 5.537 & 4.171 & 0.392 & 2.719\\
\addlinespace[3pt]\multicolumn{9}{@{}l}{\grp{Task-specialized teachers}}\\
GenEval teacher & 0.845 & --- & --- & 21.82 & 5.391 & 3.209 & 0.643 & 2.872\\
OCR teacher & --- & 0.850 & --- & 21.99 & 5.511 & 4.140 & 0.539 & 2.805\\
PickScore teacher & --- & --- & 23.08 & 23.37 & 6.151 & 4.083 & 1.140 & 3.200\\
\midrule
\ours FlowMap-OPD (ours) & \textbf{0.858} & \textbf{0.883} & \textbf{23.15} & \underline{23.45} & \textbf{6.134} & \textbf{4.275} & \underline{1.214} & \textbf{3.419}\\
\quad with deterministic map & \underline{0.825} & \underline{0.834} & 22.95 & 23.13 & 5.924 & 3.954 & 1.091 & 3.254\\ 
\quad with stochastic map & 0.821 & 0.833 & \underline{22.97} & 23.18 & \underline{5.949} & 4.028 & 1.097 & 3.252\\ 
FlowMap-GRPO & 0.733 & 0.795 & 22.92 & \textbf{23.51} & 5.806 & 4.095 & \textbf{1.281} & \underline{3.350}\\ 
MeanFlowNFT & 0.635 & 0.717 & 21.69 & 22.61 & 5.565 & 4.058 & 0.917 & 3.236\\ 
TIES~\citep{yadav2023ties} & 0.823 & 0.652 & 21.99 & 22.70 & 5.747 & 3.661 & 1.062 & 3.201\\ 
Merge (sum) & 0.638 & 0.619 & 21.48 & 22.26 & 5.650 & 2.579 & 0.956 & 2.850\\
Merge (avg) & 0.742 & 0.657 & 22.12 & 22.83 & 5.788 & \underline{4.120} & 1.040 & 3.247\\
\bottomrule\end{tabularx}
\caption{\textbf{One student with $K=4$ matches three specialists on FLUX.1-lite-8B.} FlowMap-OPD uses $\lambda_c=0.001$ at update 300; the two rows below it replace its loss by map supervision. Then: models trained directly on the three rewards, and the specialist LoRAs merged without training (Supp.~\ref{app:lora-merge}). DrawBench metrics, except PickScore, are not training rewards. Among the models below the specialists, bold marks the best and underline the second.}\label{tab:three-teacher}
\end{table*}

\section{Rollout--Kernel Separation for Flow-Map Distillation}
\label{sec:rk-theory}

\phantomsection\label{sec:rk-marginal}\label{sec:rk-roles}\label{sec:rk-replacement}%
We first show that an OPD objective can take its states from the student's native few-step rollout (Figure~\ref{fig:separation}).
Each term of the trajectory-level objective in Eq.~\eqref{eq:rk-trajectory-objective} depends on a single visited state $x_{t_i}$, not on the path that reached it. Let $d_i^{\Tcal,\theta}$ denote the state marginal at $t_i$ under a rollout $\Tcal$. By linearity of expectation, Eq.~\eqref{eq:rk-trajectory-objective} becomes
\begin{equation}
 \mathcal J_P^\Gcal(\theta)=\sum_{i=0}^{N-1}\mathbb E_{X\sim d_i^{\Tcal,\theta}}\bigl[\ell_i^P(\theta,T;X,\Gcal)\bigr].
 \label{eq:rk-original-marginals}
\end{equation}
\phantomsection\label{thm:rk-replacement}Eq.~\eqref{eq:rk-original-marginals} shows that the objective depends on the rollout only through its state marginals. For a fixed $\theta$, local comparison, and rule for sampling comparison times, two rollouts with the same state marginals give the same expected loss and, with no gradient through the rollout, the same expected update; Supp.~\ref{app:replacement} bounds the change for unequal marginals. \textbf{The kernel that compares student and teacher therefore need not be the transition that collected the states.} We call this \emph{rollout--kernel separation}.

FlowMap-OPD takes the student's native few-step rollout of Eq.~\eqref{eq:meanflowstep} as $\Tcal$ and a separately chosen local comparison $\ell$, of which Section~\ref{sec:fm-transfer} considers three.
A query is $z=(x_t,t)$ or $(x_t,t,r)$, where the destination time $r$ enters because flow-map quantities are defined on intervals, and $\rho^{\Tcal,\theta}$ is the joint distribution of states and comparison times:
\begin{equation}
 \mathcal L(\theta)=\mathbb E_{z\sim\rho^{\Tcal,\theta}}\bigl[\ell(\theta,T;z)\bigr],\;\;
 \theta\leftarrow\theta-\eta\nabla_\theta\ell(\theta,T;z).
 \label{eq:organizing-objective}
\end{equation}
The trajectory-level objective also admits a continuous-time velocity-matching form, which motivates the instantaneous-velocity supervision of Section~\ref{sec:fm-transfer}. For the SDE kernels the local KL is $\ell_i^P=h_i\|b^\theta_{t_i}(x)-b^T_{t_i}(x)\|^2/(2g_{t_i}^2)$ with $h_i=t_i-t_{i+1}$, and by Supp.~\ref{app:continuous-queries}, under grid refinement Eq.~\eqref{eq:rk-original-marginals} converges to
\begin{equation}
 \mathcal J(\theta)=\mathbb E_{t,\;x\sim d_t^{v,\theta}}\left[\frac{\|b_t^\theta(x)-b_t^T(x)\|^2}{2g_t^2}\right],
 \label{eq:rk-time-integral}
\end{equation}
where $d_t^{v,\theta}=(\psi^{v^\theta}_{1\to t})_\#p_1$ are the marginals of the flow of the student's velocity; with the score expressed through the velocity, the integrand is $w_v\|(v^\theta-v^T)(x,t)\|^2$ (Supp.~\ref{app:kernels}). This form carries over to the flow-map rollout: FlowMap-OPD takes this integrand under the native-rollout marginals $d_t^{\Tcal,\theta}$ (Eq.~\eqref{eq:vc}), which equal $d_t^{v,\theta}$ exactly when the student map is consistent with its velocity; Supp.~\ref{app:grid-propagation} bounds the deviation otherwise.

\begin{table}[t]
\centering\footnotesize

\setlength{\tabcolsep}{4pt}
\begin{tabular}{@{}llll@{}}
\toprule
Supervision & Student & Teacher & Consistency\\
\midrule
Flow map & $\psi^\theta_{t\to r}$ & $\psi^T_{t\to r}$ & none\\
Induced velocity & $V^{\theta,\mathrm{src}}$, $V^{\theta,\mathrm{dst}}$ & $v^T$ & coupled\\
\ours Decoupled velocity (ours) & $v^\theta(x,t)$ & $v^T(x,t)$ & weight $\lambda_c$\\
\bottomrule
\end{tabular}
\caption{\textbf{Only decoupled velocity supervision weights consistency separately.} All are Gaussian-kernel KLs.}\label{tab:interfaces}
\end{table}
\begin{figure*}[t]
\centering
\includegraphics[width=0.74\textwidth]{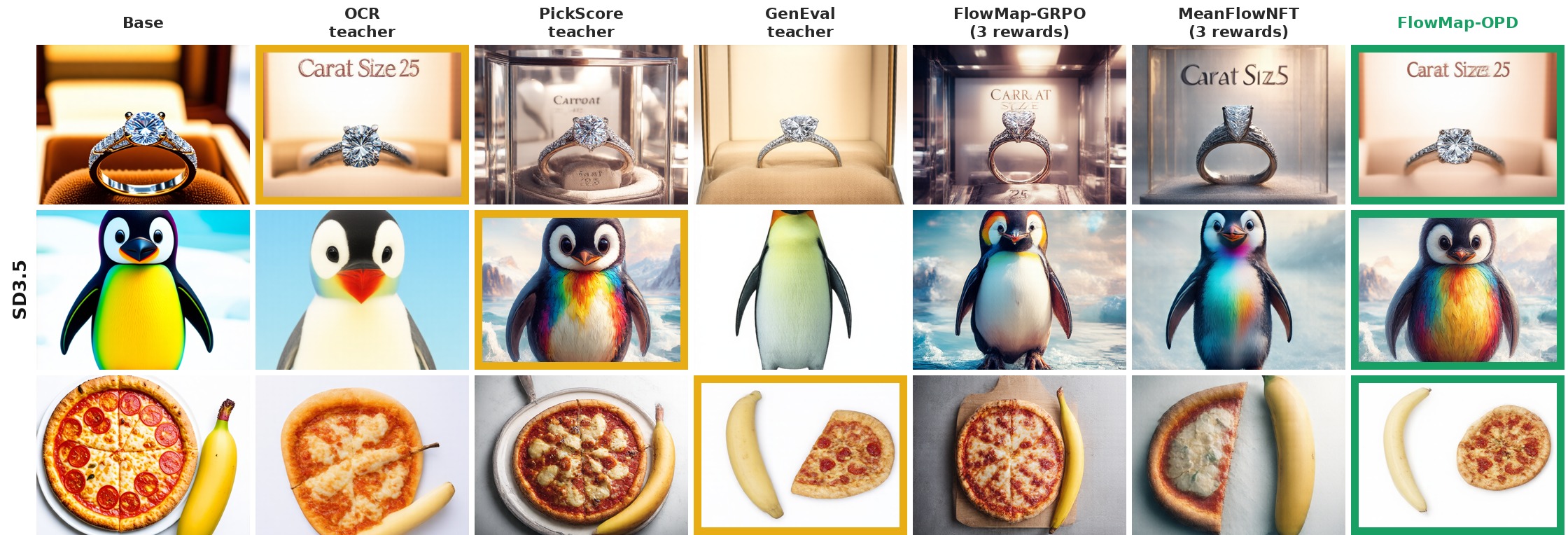}
\vspace{-2pt}
\caption{\textbf{One SD3.5-Medium student covers all three tasks.} Each row shows one prompt and noise for the base, the specialists, models trained directly on the rewards, and the FlowMap-OPD student. Rows: a text, a preference, and a two-object prompt. Gold borders mark the specialist for the row's task; green borders mark the FlowMap-OPD student.}
\label{fig:backbones}
\vspace{3pt}
\end{figure*}
\section{Supervision Choices and Consistency}
\label{sec:fm-transfer}\label{sec:fm-interfaces}

With the rollout fixed to Eq.~\eqref{eq:meanflowstep}, what remains is the comparison $\ell$. A flow-map network offers three quantities to compare: the flow map, velocities induced from it, and the instantaneous velocity (Figure~\ref{fig:separation}(3), Table~\ref{tab:interfaces}). We write each comparison as a KL between Gaussian kernels, a weighted squared difference, so the three choices differ only in the quantity compared and in how consistency enters.

\paragraph{Flow-map supervision.}\label{sec:directmap}
A deterministic jump has no density, and the KL between two point masses is infinite. Following FlowMap-GRPO~\citep{li2026flowmapgrpo}, the comparison kernel therefore adds a short anchored Langevin correction of length $\varepsilon$ and strength $\lambda$ to the jump (Supp.~\ref{app:kernels}). The resulting Gaussian KL is
\begin{equation}
 \ell_{\mathrm{map}}=w_{\mathrm{map}}\Bigl\|\Bigl(1-\tfrac{\varepsilon\lambda^2}{1-r}\Bigr)\Delta\psi-\varepsilon\lambda^2\Delta v\Bigr\|^2,
 \label{eq:finite-kernel-loss}
\end{equation}
with $w_{\mathrm{map}}=(1-r)/(4\varepsilon\lambda^2r)$, $\Delta\psi=\psi^\theta_{t\to r}(x)-\psi^T_{t\to r}(x)$, and $\Delta v=v^\theta(\psi^\theta_{t\to r}(x),r)-v^T(\psi^T_{t\to r}(x),r)$. This is the \emph{stochastic map}. As $\lambda\to0$, the KL multiplied by the kernel variance reduces to the regression $\frac12\|\psi^\theta_{t\to r}(x)-\psi^T_{t\to r}(x)\|^2$, the \emph{deterministic map} (Supp.~\ref{app:kernels}). Both require a teacher that is itself a flow map.

\paragraph{Induced-velocity supervision.}\label{sec:induced}
The identities in Eq.~\eqref{eq:fm-consistency-relations} turn the map into a velocity prediction:
\begin{equation}
 \begin{aligned}
 V^{\theta,\mathrm{src}}&=v^\theta_{t\to r}+h\bigl(\partial_tv^\theta_{t\to r}+(J_xv^\theta_{t\to r})v^\theta\bigr),\\
 V^{\theta,\mathrm{dst}}&=\partial_r\psi^\theta_{t\to r}(x).
 \end{aligned}
 \label{eq:indsrc}
\end{equation}
The source form follows MeanFlowNFT~\citep{huang2026meanflownft} and is matched to the teacher velocity through the velocity-kernel KL of Eq.~\eqref{eq:rk-time-integral}, $w_v\|V^{\theta,\mathrm{src}}-v^T(x,t)\|^2$; the destination form is matched to $v^T(\psi^\theta_{t\to r}(x),r)$ in the same way. Since the induced velocity is computed from the map through these identities, matching it optimizes the map itself: when either loss is zero at every state and interval, the student's map and velocity equal the teacher's (Supp.~\ref{app:kernels}).

\paragraph{Decoupled velocity supervision.}\label{sec:vc}
The third choice, which we propose and call \emph{decoupled velocity supervision}, uses the velocity-matching objective of Eq.~\eqref{eq:rk-time-integral}, evaluated at states from the native rollout, with a separately weighted consistency penalty:
\begin{equation}
 \Lcal_{\mathrm{velocity}}(\theta)=\E_{t,\,x\sim d_t^{\Tcal,\theta}}\bigl[w_v\|(v^\theta-v^T)(x,t)\|^2+\lambda_c\|c^\theta\|^2\bigr],
 \label{eq:vc}
\end{equation}
where we use the source residual $c^\theta=c^{\theta,\mathrm{src}}(x,t,r)$ (destination form in Supp.~\ref{app:destination-transfer}) and the expectation is also over the destination time $r<t$. For MeanFlow, $c^{\theta,\mathrm{src}}=v^\theta_{t\to r}+h(\partial_tv^\theta_{t\to r}+(J_xv^\theta_{t\to r})v^\theta)-v^\theta$ with $v^\theta=v^\theta(x,t)$. The first term matches the teacher's instantaneous velocity at student-visited states, while the second enforces consistency between the learned velocity and the deployed flow map, so that the map is optimized as well.

\begin{algorithm}[t]
\caption{FlowMap-OPD (decoupled velocity)}
\label{alg:vc}
\begin{algorithmic}[1]
\Require student $\theta$, frozen teacher $T$, grid $\Gcal$, weight $\lambda_c\geq0$
\For{each training round}
 \State Roll out $x_{t_{i+1}}=\psi^\theta_{t_i\to t_{i+1}}(x_{t_i})$ from $x_{t_0}\sim p_1$; store the visited $(x_{t_i},t_i)$ as detached queries
 \For{each update on these states}
  \State Sample $(x,t)$ and a destination $r<t$
  \State $\ell\gets w_v\|(v^\theta-v^T)(x,t)\|^2+\lambda_c\widetilde\ell_c$ \Comment{Eq.~\eqref{eq:stopped-consistency}}
  \State $\theta\gets\theta-\eta\nabla_\theta\ell$
 \EndFor
\EndFor
\State \Return the student, sampled with its native few-step rollout
\end{algorithmic}
\end{algorithm}

\subsection{Decoupling velocity and consistency}\label{sec:decouple}
At inference the student samples with its map, so the map error is what supervision must control. It is bounded by the two errors in Eq.~\eqref{eq:vc}, the velocity error $e^\theta=v^\theta-v^T$ and the consistency residual $c^\theta=c^{\theta,\mathrm{src}}$, as shown next.
\begin{theorem}[Error transfer to the deployed map]
\label{thm:main-transfer}
Let $\psi^T$ be the flow of the teacher velocity $v^T$. Fix $x$ and $r<t$, set $h=t-r$, and let $x_q=\psi^T_{t\to q}(x)$. If for all $q\in(r,t]$ we have $\|c^{\theta}(x_q,q,r)\|\leq\epsilon_c$, $\|e^\theta(x_q,q)\|\leq\epsilon_v$, and $\|J_x\psi^\theta_{q\to r}(x_q)-\Id\|_{\mathrm{op}}\leq M(q-r)$, then
\begin{equation}
 \|\psi^\theta_{t\to r}(x)-\psi^T_{t\to r}(x)\|\leq h\,\epsilon_c+\bigl(h+Mh^2/2\bigr)\epsilon_v.
 \label{eq:main-local-error}
\end{equation}
\end{theorem}
The first term accumulates the map's inconsistency with its own velocity, the second the velocity's mismatch with the teacher; for MeanFlow, $M$ bounds $\|J_xv^\theta_{q\to r}\|$. The proof is in Supp.~\ref{app:local-error}; Supp.~\ref{app:destination-transfer}--\ref{app:population-transfer} give the destination form, the propagation through the rollout, and the extension to training losses. Controlling each error thus controls the map, and different forms of velocity supervision do so differently.

\textbf{Velocity matching alone controls only $e^\theta$.} It sees only the diagonal $r=t$ of the student, and many maps share one diagonal. For example, take a stationary teacher, $v^T=0$: a student with $v^\theta_{t\to r}(x)=-(t-r)a$, $a\neq0$, matches the diagonal exactly, yet $\psi^\theta_{t\to r}(x)=x+(t-r)^2a\neq x$; only the residual $c^{\theta}=-2(t-r)a$ detects the error (Supp.~\ref{app:diagonal-example}). In Figure~\ref{fig:mechanism}b, the $\lambda_c=0$ student learns the velocity but its $K=3$ map does not follow.

\textbf{Induced supervision controls only their sum.} The source-induced prediction of Eq.~\eqref{eq:indsrc} satisfies $V^{\theta,\mathrm{src}}-v^T(x,t)=e^\theta+c^\theta$, so its loss is
\begin{equation}
 \|V^{\theta,\mathrm{src}}-v^T(x,t)\|^2=\|e^\theta\|^2+\|c^\theta\|^2+2\langle e^\theta,c^\theta\rangle.
 \label{eq:source-error-decomposition}
\end{equation}
The loss is small whenever the sum $e^\theta+c^\theta$ is small. Because of the cross term in Eq.~\eqref{eq:source-error-decomposition}, this can happen while both errors are large and nearly opposite, so a small loss bounds neither $e^\theta$ nor $c^\theta$. Their relative weight is also fixed.

\textbf{The decoupled loss controls each with its own weight.} Eq.~\eqref{eq:vc} penalizes $w_v\|e^\theta\|^2+\lambda_c\|c^\theta\|^2$, with no cross term and a free $\lambda_c\geq0$. A larger $\lambda_c$ lowers the consistency residual and raises the velocity mismatch. Theorem~\ref{thm:main-transfer} does not prescribe their relative weight; empirically, small but nonzero weights transfer best (Section~\ref{sec:imagenet}).

\begin{table*}[t]
\centering\small

\setlength{\tabcolsep}{4pt}\renewcommand{\arraystretch}{1.0}
\begin{tabularx}{\textwidth}{@{}l*{8}{>{\centering\arraybackslash}X}@{}}
\toprule
& \multicolumn{3}{c}{\textbf{Task scores} $\uparrow$} & \multicolumn{5}{c}{\textbf{Quality preservation (DrawBench)} $\uparrow$}\\
\cmidrule(lr){2-4}\cmidrule(l){5-9}
Model / supervision & GenEval & OCR & PickScore & PickScore & Aesthetic & DeQA & ImgRwd & UniRwd\\
\midrule
Base & 0.726 & 0.569 & 21.69 & 22.50 & 5.403 & 4.107 & 1.159 & 3.237\\
\addlinespace[3pt]\multicolumn{9}{@{}l}{\grp{Task-specialized teachers}}\\
GenEval teacher & 0.868 & --- & --- & 22.16 & 5.140 & 3.975 & 0.904 & 3.002\\
OCR teacher & --- & 0.912 & --- & 21.73 & 5.056 & 3.784 & 0.433 & 2.792\\
PickScore teacher & --- & --- & 23.52 & 23.69 & 6.170 & 4.118 & 1.307 & 3.299\\
\midrule
\ours FlowMap-OPD (ours) & \textbf{0.872} & \textbf{0.905} & \textbf{23.40} & \underline{23.51} & \textbf{6.011} & 4.083 & \underline{1.262} & \textbf{3.332}\\ 
FlowMap-GRPO & 0.713 & 0.824 & \underline{23.11} & \textbf{23.54} & \underline{5.984} & 4.099 & \textbf{1.265} & 3.303\\
MeanFlowNFT & 0.732 & \underline{0.865} & 22.18 & 22.95 & 5.797 & 3.818 & 1.154 & 3.241\\
TIES~\citep{yadav2023ties} & \underline{0.816} & 0.779 & 22.49 & 22.89 & 5.513 & \underline{4.176} & 1.156 & 3.265\\ 
Merge (sum) & 0.667 & 0.543 & 21.54 & 21.87 & 5.387 & 3.411 & 0.567 & 2.721\\
Merge (avg) & 0.789 & 0.755 & 22.38 & 22.93 & 5.476 & \textbf{4.242} & 1.173 & \underline{3.307}\\
\bottomrule\end{tabularx}
\caption{\textbf{One student matches three specialists on SD3.5-Medium} with a pretrained flow-map checkpoint, compared with one model trained directly on the three rewards and with merging the three specialist LoRAs (Supp.~\ref{app:lora-merge}); columns as in Table~\ref{tab:three-teacher}. Higher is better; best in bold, second underlined among the models below the specialists.}\label{tab:three-teacher-sd}
\end{table*}

\subsection{Algorithm and implementation}\label{sec:training}\label{sec:gradients}
Algorithm~\ref{alg:vc} summarizes training. In each round the student generates native rollouts with Eq.~\eqref{eq:meanflowstep} (time grids in Supp.~\ref{app:random-grids}); the visited states are detached and reused for several updates before new rollouts are collected.

The source residual in the consistency term requires one Jacobian--vector product of $v^\theta_{t\to r}$ with respect to $(x,t,r)$, in the direction $(v^\theta,1,0)$. Following MeanFlow~\citep{geng2025meanflow}, we stop gradients through this derivative term (Supp.~\ref{app:stopped-jvp}):
\begin{equation}
 \widetilde\ell_c=\bigl\|v^\theta_{t\to r}+h\,\sg\bigl[\partial_tv^\theta_{t\to r}+(J_x v^\theta_{t\to r})v^\theta\bigr]-v^\theta\bigr\|^2.
 \label{eq:stopped-consistency}
\end{equation}
The loss has the same forward value as $\|c^{\theta,\mathrm{src}}\|^2$, the square of the residual norm in Theorem~\ref{thm:main-transfer}; only the gradient changes, and the measured residual is still reduced (Table~\ref{tab:consistency-errors}). Backpropagating through the Jacobian--vector product instead diverges (Section~\ref{sec:imagenet}), because its gradient contains second derivatives of the network~\cite{geng2025meanflow}.

\section{Experiments}\label{sec:experiments}
\begin{table}[t]\centering\footnotesize
\setlength{\tabcolsep}{3pt}
\begin{tabularx}{\columnwidth}{@{}Yrrrr@{}}\toprule
Supervision & \multicolumn{2}{c}{MMD} & Classifier & DINO\\
 & FID $\downarrow$ & FID$_{50\mathrm{k}}\downarrow$ & $\log p_y\uparrow$ & Cosine $\uparrow$\\\midrule
B/2 base & 22.68 & 13.34 & -0.545 & 0.721\\
XL/2 teacher & 14.68 & 3.86 & -0.453 & 0.808\\
\midrule\multicolumn{5}{@{}l}{\grp{Flow-map supervision}}\\
Stochastic map & 27.89 & 15.03 & -0.962 & 0.700\\ 
Deterministic map & 32.84 & 19.20 & -0.696 & 0.759\\
\midrule\multicolumn{5}{@{}l}{\grp{Induced-velocity supervision}}\\
Source induced & \underline{20.04} & \underline{8.22} & \underline{-0.512} & 0.752\\
Destination induced & 42.74 & 28.66 & -0.568 & \textbf{0.790}\\
\midrule\multicolumn{5}{@{}l}{\grp{Decoupled velocity supervision}}\\
$\lambda_c=0$ & 21.72 & 9.56 & -0.548 & 0.760\\
\ours $\lambda_c=0.01$ (ours) & \textbf{18.23} & \textbf{6.74} & \textbf{-0.472} & \underline{0.781}\\
$\lambda_c=0.1$ & 21.05 & 9.21 & -0.538 & 0.768\\
$\lambda_c=1$ & 43.07 & 29.56 & -0.671 & 0.020\\
\bottomrule\end{tabularx}
\caption{\textbf{Decoupled velocity supervision transfers best.} XL/2-to-B/2 students at epoch 300 with native $K=3$ sampling. Each reward is a separate teacher and student, evaluated with that teacher's metric; FID$_{50\mathrm{k}}$ uses 50{,}000 samples against all real training images of the 100 classes. Stochastic map: best $\lambda$ per column. Among students: best in bold, second underlined.}\label{tab:imagenet}
\end{table}

We first compare the supervision choices on ImageNet (Section~\ref{sec:imagenet}), then consolidate three text-to-image specialists into one few-step student (Section~\ref{sec:t2i}), and finally measure what rollout--kernel separation provides (Section~\ref{sec:separation-exp}). Main experiments use MeanFlow students; consistency-model students are in Supp.~\ref{app:imagenet-results}. All runs set $w_v=1$ and take $r$ as the next grid time with probability $0.5$, else uniform in $(0,t)$.

\subsection{Design-space exploration on ImageNet}\label{sec:imagenet}
\begin{figure}[t]
\centering
\includegraphics[width=\columnwidth]{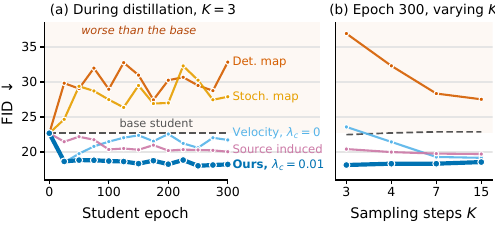}
\caption{\textbf{Velocity supervision improves the student, and consistency carries the gain to the few-step map} (MMD teacher). (a)~Students during distillation, sampled with $K=3$. (b)~The epoch-300 checkpoints sampled with different $K$ on the time grid of FlowMap-GRPO~\citep{li2026flowmapgrpo} (numbers in Supp.~\ref{app:imagenet-results}, Table~\ref{tab:step-count}).}
\label{fig:mechanism}
\end{figure}

Which quantity should the teacher supervise? We compare the three supervision choices by distilling XL/2 MeanFlow teachers
into B/2 students with the same native rollout and number of optimization steps.

\paragraph{Setup.}
For each reward (MMD distribution matching, classifier target log-probability,
or DINO cosine similarity), a student initialized from the same
pretrained B/2 model distills a teacher post-trained with FlowMap-GRPO~\citep{li2026flowmapgrpo}. Each setting compares twelve configurations (the main ones in Table~\ref{tab:imagenet}, all in Supp.~\ref{app:imagenet-results}), all with deterministic $K=3$ rollouts, at epoch 300 (each epoch collects 256 rollouts and takes four updates); $\lambda_c$ is selected by FID on a validation split disjoint from the test images. The configurations share the pretrained student, the teacher, the seed, and the evaluation noises, so rows differ only in the supervised quantity. The MMD reward and the evaluation setting are those of Miao et al.~\citep{miao2024diverse}: ImageNet classes 0--99, on which the reward is defined and the teachers are trained, 5,000 generated images, and FID, precision, and recall against 2,000 held-out real images. At the standard 50,000 samples, the ordering is unchanged and our method lowers FID from 13.34 to 6.74 (Supp.~\ref{app:imagenet-protocol}, large-sample FID).

\paragraph{Supervision comparison.}
Decoupled velocity supervision transfers best, and flow-map supervision, although it targets the object used at inference, transfers worst (Table~\ref{tab:imagenet}, Figure~\ref{fig:mechanism}a; curves for all rewards in Figure~\ref{fig:imagenet-design-selection}). (i)~With $\lambda_c=0.01$ our method ranks first for MMD and classifier transfer and second for DINO, lowering FID from 22.68 to 18.23 (Table~\ref{tab:imagenet}). (ii)~Flow-map supervision, with each choice at its best configuration, is the weakest on all three rewards, and some results even end below the base student. (iii)~Induced velocities lie in between: source-induced supervision, MeanFlow's own training loss with the teacher's velocity as target, ranks second for MMD and classifier transfer and ends 1.8 FID higher; destination-induced supervision is unstable. Velocity supervision also helps consistency models (Table~\ref{tab:cm-transfer}, Supp.~\ref{app:imagenet-results}): a fixed-endpoint student distilled from the MMD teacher reaches FID 22.45 with source-induced velocity supervision, against 40.1--41.0 when the endpoint map is supervised.

Map supervision is weak because its target carries the difference between the pretrained teacher and student maps. We vary how much of this difference enters the regression target, with capacity and optimizer fixed: precision falls monotonically from 0.80 to 0.54 (MMD teacher; Table~\ref{tab:map-target}, Supp.~\ref{app:map-target}). Capacity is not the cause: a fourfold larger adapter leaves both map regression and our method unchanged (Supp.~\ref{app:map-target}).

\paragraph{Consistency ablations.}
The consistency penalty carries the learned velocity to the deployed map. To test whether this holds at other step counts, we sample the same checkpoints with $K=3$, 4, 7, and 15 on the time grid of FlowMap-GRPO~\citep{li2026flowmapgrpo} (Figure~\ref{fig:mechanism}b). Without consistency, the MMD student improves on the base at $K=15$ (19.17) but not at $K=3$ (23.58; Table~\ref{tab:step-count}): the velocity has learned from the teacher, model difference included, but the long-range map has not followed, as Section~\ref{sec:decouple} describes for velocity matching alone. With $\lambda_c=0.01$, FID is stable across all four $K$ (18.14 at $K=3$), so the $K=3$ sampler matches the quality of finer integration. Source-induced supervision is as stable at a higher FID, and map regression improves with $K$ yet stays above the base. The weight must stay small, since a larger one trades velocity error for consistency (Section~\ref{sec:decouple}): a consistency weight of one falls below the base on all three rewards, and $\lambda_c=10$ diverges (Table~\ref{tab:imagenet-step-comparison}). Without the stop-gradient of Eq.~\eqref{eq:stopped-consistency}, training diverges (FID 82.3 against 18.23; Supp.~\ref{app:imagenet-results}). Text-to-image distillation shows the same trade-off: increasing $\lambda_c$ from 0 to 1 reduces the measured consistency residual from 0.800 to 0.179 and raises the velocity mismatch from 0.017 to 0.171, and a small weight is again best (Tables~\ref{tab:consistency-errors} and~\ref{tab:consistency-t2i}).

\subsection{Text-to-image specialist consolidation}\label{sec:t2i}
\paragraph{Setup.}
\begin{figure}[t]
\centering
\includegraphics[width=\columnwidth]{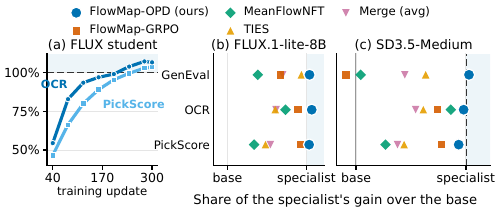}
\caption{\textbf{One distilled student reaches the specialists on every task; reward-trained and merged models do not.} Axes show the share of a specialist's gain over the base. (a)~The FLUX student during distillation. (b, c)~Final task scores of Tables~\ref{tab:three-teacher} and~\ref{tab:three-teacher-sd}.}
\label{fig:consolidation-gain}
\end{figure}
We next apply the method identified on ImageNet to a practical task: consolidating an OCR teacher, a PickScore teacher, and a GenEval teacher~\citep{ghosh2023geneval,kirstain2023pickapic} into one native MeanFlow student.
The base is FLUX.1-lite-8B with the public MeanFlow checkpoint of T2I-Distill~\cite{pu2025fewstep},
at $512\times512$ and guidance 3.5. Teachers and student are LoRAs; the teachers are post-trained with FlowMap-GRPO~\citep{li2026flowmapgrpo}.
Prompts are sampled 1:1:1 across tasks and routed to their teachers.
Rollouts use $K=4$ and
instantaneous-velocity supervision with $\lambda_c=0.001$, selected by the rewards on training prompts (Supp.~\ref{sec:consistency-experiment}). Each round collects 1,152 rollouts on eight H100 GPUs and is used for two updates.
We use AdamW with learning rate $10^{-4}$ and train for 300 updates.
EMA students are evaluated every 40 updates on the full OCR and
PickScore test sets (Figure~\ref{fig:consolidation-gain}a) and on GenEval at the end.

\paragraph{Results.}
The same student checkpoint acquires all three specialist capabilities in 300 updates and exceeds each specialist on its own task at the same number of sampling steps (OCR 0.883 against 0.850; Table~\ref{tab:three-teacher}, Figure~\ref{fig:consolidation-gain}). Here teachers and student share one pretrained base, which removes the pretrained-map difference of Section~\ref{sec:imagenet}, and map supervision recovers, yet it still trails velocity supervision on all three tasks (OCR 0.834 against 0.883; Table~\ref{tab:three-teacher}). Figure~\ref{fig:capability-comparison} shows the same on individual prompts (more in Supp.~\ref{app:paired-gallery}).

\paragraph{Comparison with reward training.}
We compare FlowMap-OPD with FlowMap-GRPO~\citep{li2026flowmapgrpo} and MeanFlowNFT~\citep{huang2026meanflownft}, which train one model directly on the three rewards with the same flow-map backbone and prompt mixture, each given the rollout budget of 1{,}500 GRPO updates used by FlowMap-GRPO~\citep{li2026flowmapgrpo} (750 rounds of 1{,}152 rollouts, against 150 for FlowMap-OPD). FlowMap-OPD instead distills trained specialists, as Flow-OPD and DiffusionOPD do; the FLUX teachers took 800, 1{,}000, and 1{,}500 GRPO updates in the setting of FlowMap-GRPO~\citep{li2026flowmapgrpo}.

The comparison is therefore about consolidation: whether one model can hold the reward-specific performance of several specialists. Both joint models trail each specialist on its own task on FLUX and SD3.5 (Tables~\ref{tab:three-teacher} and~\ref{tab:three-teacher-sd}), and on SD3.5 FlowMap-GRPO lowers GenEval below the base: one model optimized on several rewards trades them against each other, whereas each teacher in OPD supervises only its own task. In a two-reward experiment, FlowMap-GRPO flattens below both specialists by update 1{,}400, while the distilled student is above both at update 300 (Figure~\ref{fig:two-teacher-grpo}, Supp.~\ref{app:two-task}). The FLUX student also exceeds the base on every DrawBench metric.

We further apply our method to SD3.5-Medium with the public AnyFlow~\citep{gu2026anyflow} flow-map checkpoint released with MeanFlowNFT~\citep{huang2026meanflownft}, $K=4$, and the FLUX $\lambda_c=0.001$, untuned. The specialists are post-trained with FlowMap-GRPO; the student exceeds the GenEval specialist and is within 1\% of the other two (Table~\ref{tab:three-teacher-sd}; samples in Figure~\ref{fig:backbones}). FlowMap-OPD also applies to consistency models, which expose only an endpoint: on FLUX, one sCM student recovers 77--95\% of the gains of three sCM specialists (Supp.~\ref{app:scm-t2i}).

\subsection{Benefits of rollout--kernel separation}\label{sec:separation-exp}
The student's own rollout states matter: replacing the native rollout states by DMD-style re-noised student samples~\citep{yin2024dmd}, all else equal, raises the student's ImageNet FID from 18.23 to 22.56, back to the base's 22.68 (Supp.~\ref{app:state-source}).

Training rollouts are short. A $K=4$ rollout takes five network evaluations, and fewer jumps suffice: on SD3.5-Medium, students trained with $K=3$ or $K=2$ rollouts match the $K=4$ student when all are sampled with $K=4$ after 300 updates, and $K=10$ brings no gain at twice the evaluations (Table~\ref{tab:rollout}).

\begin{table}[t]
\centering\footnotesize
\setlength{\tabcolsep}{4.5pt}
\begin{tabular}{@{}lccccc@{}}
\toprule
Training rollout & NFE & Time & GenEval & OCR & PickScore\\
\midrule
$K=10$ & 11 & 40.0\,s & 0.864 & \underline{0.901} & \textbf{23.42}\\ 
\ours $K=4$ (default) & 5 & 22.9\,s & \textbf{0.872} & \textbf{0.905} & \underline{23.40}\\ 
$K=3$ & 4 & 19.6\,s & 0.865 & 0.898 & 23.37\\
$K=2$ & 3 & 16.2\,s & \underline{0.866} & 0.897 & 23.34\\
\bottomrule
\end{tabular}
\caption{\textbf{Fewer rollout jumps suffice.} SD3.5-Medium, update 300, sampled with $K=4$; NFE per rollout, time to collect 1{,}152 rollouts (8 H100).}\label{tab:rollout} 
\end{table}

\begin{table}[t]
\centering\footnotesize
\setlength{\tabcolsep}{3.5pt}
\begin{tabular}{@{}llcccc@{}}
\toprule
Teachers & Student & Train NFE & GenEval & OCR & PickScore\\
\midrule
\ours Few-step & Few-step & 5 & \underline{0.872} & \textbf{0.905} & \textbf{23.40}\\ 
Few-step & Multi-step & 5 & 0.837 & \underline{0.897} & \underline{23.20}\\ 
\midrule
Multi-step & Few-step & 5 & \textbf{0.895} & 0.828 & 22.83\\ 
\multicolumn{2}{@{}l}{Flow-OPD~\citep{fang2026flowopd}, same teachers} & 20 & 0.853 & 0.790 & 22.32\\ 
\midrule\multicolumn{6}{@{}l}{\grp{The same two students with fewer sampling steps ($K=3$, then $K=2$)}}\\
\ours Multi-step & Few-step & 5 & \textbf{0.862} & \textbf{0.760} & \textbf{22.59}\\ 
\multicolumn{2}{@{}l}{Flow-OPD~\citep{fang2026flowopd}, same teachers} & 20 & 0.734 & 0.577 & 21.70\\ 
\ours Multi-step & Few-step & 5 & \textbf{0.741} & \textbf{0.532} & \textbf{21.78}\\ 
\multicolumn{2}{@{}l}{Flow-OPD~\citep{fang2026flowopd}, same teachers} & 20 & 0.355 & 0.157 & 20.11\\ 
\bottomrule
\end{tabular}
\caption{\textbf{Teachers and students need not be flow maps} (SD3.5-Medium, update 300, sampled with $K=4$ unless noted; NFE per rollout). Rows with multi-step teachers share the public teachers.}\label{tab:multistep-teachers}
\end{table}

Teachers and students need not be flow maps. Every flow model has a velocity, so the teacher can be a multi-step model. With the public multi-step Flow-GRPO specialists as teachers, FlowMap-OPD evaluates the teacher velocity at each visited state and trains a few-step student with five evaluations per rollout; Flow-OPD with the same teachers spends twenty on a multi-step student, which is weaker when both are sampled with $K=4$ and falls further behind with fewer steps (Table~\ref{tab:multistep-teachers}; Supp.~\ref{app:opd-baseline}). Consistency carries the supervised velocity to the map, so the student need not be a pretrained flow map either: initialized from the multi-step SD3.5-Medium weights with an added second time input, it comes close to the few-step student in 300 updates (Table~\ref{tab:multistep-teachers}, second row; Supp.~\ref{app:multistep-init}).

\section{Related Work}\label{sec:related}
\paragraph{On-policy distillation.}
OPD originates in policy distillation~\citep{czarnecki2019distilling}; for language models, GKD, MiniLLM, and DistiLLM make the sampling distribution and the objective explicit choices~\citep{agarwal2023opd,gu2024minillm,ko2024distillm}, and Vision-OPD applies on-policy self-distillation to multimodal LLMs~\citep{yuan2026visionopd}. In continuous generation, DiffusionOPD and Flow-OPD compare local conditional transitions of multi-step samplers~\citep{li2026diffusionopd,fang2026flowopd}, and Any-OPD bridges heterogeneous flow-matching teachers and students in a shared visual representation~\citep{fu2026anyopd}. These methods use one transition for sampling and comparison; we separate the two for flow-map students, whose sampler is a long-range map while supervision can act on local dynamics.

\paragraph{Flow maps and their post-training.}
MeanFlow separates average from instantaneous velocity and derives the identity linking them~\citep{geng2025meanflow}; consistency models learn endpoint predictions that agree along a trajectory~\citep{song2023consistency,lu2025simplifying}; see also~\citep{boffi2024fmm,sabour2025ayf}. For reward post-training, MeanFlowNFT optimizes an induced instantaneous predictor~\citep{huang2026meanflownft} and FlowMap-GRPO builds anchored stochastic transitions around the map~\citep{li2026flowmapgrpo}. We compare them as distillation targets.

\paragraph{Few-step distillation.}
DMD and SiD also train a few-step generator on its own samples~\citep{yin2024dmd,yin2024dmd2,zhou2024sid}, through a score difference from an auxiliary network; FlowMap-OPD compares velocities state by state and needs none. DiffusionNFT and DMSampler separate sampler from objective for diffusion reinforcement learning~\citep{zheng2026diffusionnft,liu2026dmsampler}; we do so for OPD.

\section{Conclusion}\label{sec:conclusion}
FlowMap-OPD brings on-policy distillation to few-step flow-map generators with one rule: sample with the flow map, learn the velocity. Decoupled velocity supervision transfers best across pretrained models and consolidates three text-to-image specialists into one student at their level. \textbf{Limitations.} FlowMap-OPD starts from trained specialists, and $\lambda_c$ is set per model family. More broadly, what a few-step generator samples with need not be what it is taught through.

{\small
\bibliographystyle{ieeenat_fullname}
\bibliography{references}
}

\clearpage
\onecolumn
\appendix
\section*{Appendix contents}
\pdfbookmark[0]{Appendix contents}{appendix-contents}
The appendices first give additional ImageNet and text-to-image results and qualitative comparisons, then proofs for the separated objective and
consistency results, kernel constructions, and implementation details.
\makeatletter
\begingroup
\setcounter{tocdepth}{2}
\small
\@starttoc{apc}
\endgroup
\let\appendixaddcontentsline\addcontentsline
\renewcommand{\addcontentsline}[3]{%
  \def\appendixcontentsfile{#1}\def\maintocfile{toc}%
  \ifx\appendixcontentsfile\maintocfile
    \appendixaddcontentsline{apc}{#2}{#3}%
  \else
    \appendixaddcontentsline{#1}{#2}{#3}%
  \fi}
\makeatother

\raggedbottom
\section{Additional ImageNet Results}\label{app:imagenet-results}\label{app:imagenet-protocol}
This section gives the full setting and results behind Section~\ref{sec:imagenet}: the curves for all rewards and configurations, the large-sample FID, and the controls on the consistency gradient, the state source, and the map target.
\begin{figure*}[!htbp]\centering
\includegraphics[width=\textwidth]{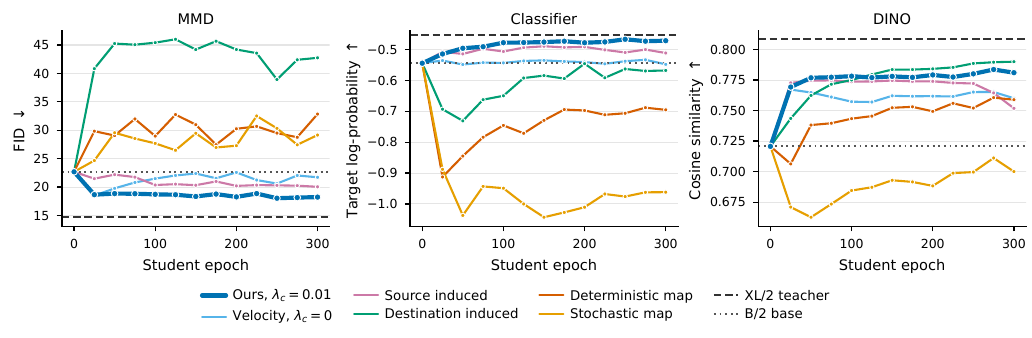}
\caption{\textbf{Cross-capacity transfer under three teacher rewards.} MMD-setting FID, classifier target log-probability, and DINO cosine similarity of B/2 students under native generation with $K=3$, for representative configurations of each supervision choice. Dashed lines are the frozen XL/2 teachers; all configurations and curves at $K=4$ are in Figure~\ref{fig:imagenet-three-all}.}
\label{fig:imagenet-design-selection}
\end{figure*}

\paragraph{Teacher and student.}
Both models use the MeanFlow parameterization. We train three XL/2
teachers by strictly reproducing FlowMap-GRPO~\citep{li2026flowmapgrpo}, each adapted to a separate reward: the maximum mean discrepancy reward of Miao et al.~\citep{miao2024diverse} rewards distribution matching (MMD);
classifier reward maximizes the target-class log-probability under a frozen
ConvNeXt-Base (Classifier); and the reward maximizes cosine similarity
to the mean class feature computed from 50 real training images using frozen
DINOv2-small features (DINO).
Each setting tests transfer from a separately adapted teacher.
We study asymmetric distillation from XL/2 teachers to B/2 students,
all initialized from the same pretrained B/2 model. All 36 configurations are compared at epoch 300.

\paragraph{Supervision comparison.}
Each reward setting compares the same twelve B/2 configurations at epoch 300:
velocity supervision
with $\lambda_c\in\{0,0.01,0.1,1,10\}$, stochastic flow-map supervision
with $\lambda\in\{0.1,0.2,0.4,0.8\}$, deterministic map regression,
and source- and destination-induced velocity supervision.
Each epoch collects 256 trajectories and takes
four updates of effective batch size 256. We use AdamW with learning
rate $10^{-4}$, weight decay $10^{-4}$, gradient clipping at 1,
and fresh rank-64 LoRA adapters with scaling 128.
For off-diagonal queries, the destination is the native next time with
probability $0.5$, otherwise uniform in $(0,t)$. Following MeanFlow, the induced-source
recipe uses $75\%$ diagonal queries; both induced recipes use adaptive
weighting and stopped targets. Consistency uses the stopped-JVP update.
Stochastic-map comparison uses $\varepsilon=0.01$ and its kernel-KL
weight, with deterministic comparison when $r\leq\varepsilon$.
Squared errors are averaged over latent coordinates. Instantaneous velocity uses $w_v=1$; deterministic map regression retains the $h^2/2$ displacement scaling. Induced losses use unit outer weighting and per-sample adaptive normalization $\ell/\sg(\ell+10^{-6})$. The stochastic-map velocity correction is detached.

\paragraph{Evaluation.}
We follow the evaluation setting of Miao et al.~\citep{miao2024diverse}, whose MMD reward trains the MMD teacher: 5,000 generated images across ImageNet classes 0--99, with
50 images per class. FID, precision and recall are computed against
2,000 real test images; a separate set of 3,000 real images is used
for validation. The consistency weight is selected on this validation set with the MMD teacher (validation FID 15.56 at $\lambda_c=0.01$, against 19.07 at 0 and 18.23 at 0.1) and then used for all three rewards.
All student main results use epoch 300. Evaluation at $K=3$ uses
the native uniform grid $[1,0.75,0.5,0.25,0]$. To test whether distillation gains persist
when the number of sampling steps changes, we also evaluate $K=4$
generation using $\operatorname{linspace}(1,0.05,5)$ followed by zero. Table~\ref{tab:imagenet-step-comparison} compares $K=3$ and $K=4$ results;
Figure~\ref{fig:imagenet-three-all} shows the training curves. Following the time discretization in FlowMap-GRPO~\citep{li2026flowmapgrpo}, we also evaluate the same frozen checkpoints at $K=3$, 4, 7 and 15 using $\operatorname{linspace}(1,0.05,N)$ followed by zero.
\begin{figure}[!htbp]\centering
\includegraphics[width=0.76\textwidth]{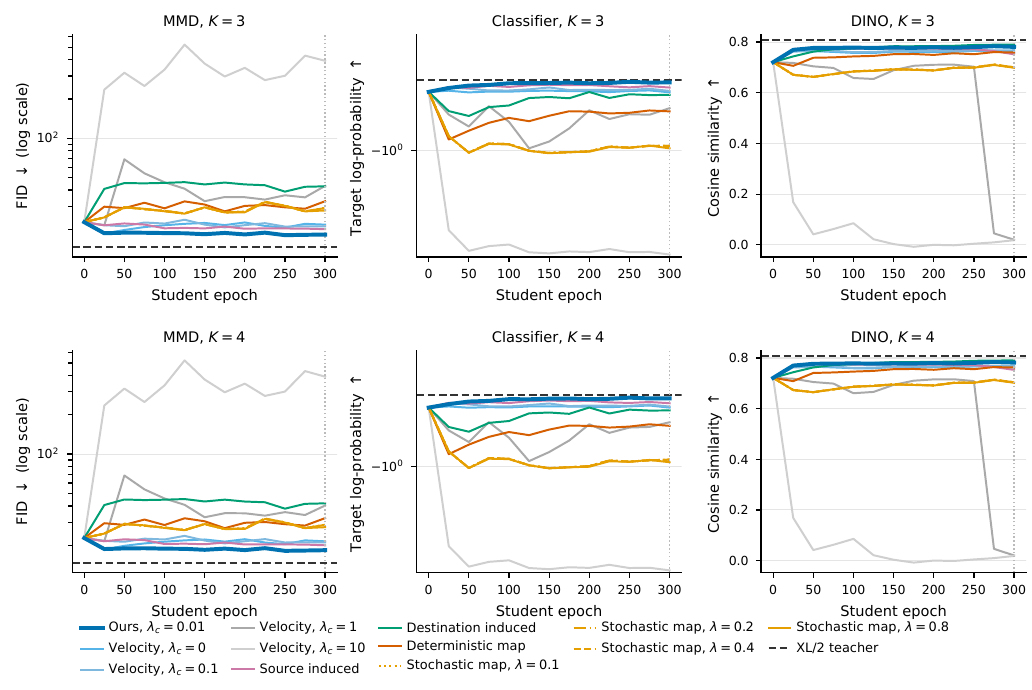}
\caption{\textbf{All twelve configurations under the three teacher rewards.} Supervision configurations for each of the three teacher
rewards, at $K=3$ (top) and $K=4$ (bottom). Markers
are recorded evaluations, dashed horizontal lines are the frozen teachers,
and the dotted vertical line marks the common epoch-300 comparison.
MMD FID uses a logarithmic scale; classifier log-probability uses a symmetric
logarithmic scale with a linear region between $-1$ and $1$, retaining
unstable configurations. Curves end at the last available evaluation.}
\label{fig:imagenet-three-all}
\end{figure}
\begin{table}[!htbp]\centering\footnotesize
\caption{\textbf{The ordering is the same at $K=3$ and $K=4$.} Reward-transfer results at epoch 300. Each reward setting reports its corresponding metric; stochastic map reports its best $\lambda$ per column. Among students: best in bold, second underlined.}
\label{tab:imagenet-step-comparison}
\setlength{\tabcolsep}{7pt}
\begin{tabular}{@{}lrrrrrr@{}}\toprule
 & \multicolumn{2}{c}{MMD: FID $\downarrow$} & \multicolumn{2}{c}{Classifier: $\log p_y\uparrow$} & \multicolumn{2}{c}{DINO: Cosine $\uparrow$}\\
\cmidrule(lr){2-3}\cmidrule(lr){4-5}\cmidrule(lr){6-7}
Supervision & $K=3$ & $K=4$ & $K=3$ & $K=4$ & $K=3$ & $K=4$\\\midrule
B/2 base & 22.68 & 22.72 & -0.5448 & -0.5451 & 0.7207 & 0.7215\\
XL/2 teacher & 14.68 & 14.73 & -0.4526 & -0.4500 & 0.8084 & 0.8096\\
\midrule
Stochastic map & 27.89 & 27.52 & -0.9619 & -0.9469 & 0.7004 & 0.7040\\
Deterministic map & 32.84 & 32.31 & -0.6956 & -0.6860 & 0.7589 & 0.7624\\
Source induced & \underline{20.04} & \underline{19.99} & \underline{-0.5118} & \underline{-0.5113} & 0.7519 & 0.7531\\
Destination induced & 42.74 & 41.92 & -0.5682 & -0.5671 & \textbf{0.7901} & \textbf{0.7904}\\
Velocity, $\lambda_c=0$ & 21.72 & 21.44 & -0.5484 & -0.5453 & 0.7601 & 0.7625\\
\ours Velocity, $\lambda_c=0.01$ (ours) & \textbf{18.23} & \textbf{18.33} & \textbf{-0.4720} & \textbf{-0.4737} & \underline{0.7811} & \underline{0.7817}\\
Velocity, $\lambda_c=0.1$ & 21.05 & 21.01 & -0.5378 & -0.5371 & 0.7679 & 0.7684\\
Velocity, $\lambda_c=1$ & 43.07 & 40.33 & -0.6710 & -0.6582 & 0.0202 & 0.0202\\
Velocity, $\lambda_c=10$ & 392.78 & 392.19 & -7.7761 & -7.7774 & 0.0183 & 0.0183\\
\bottomrule\end{tabular}\end{table}

\paragraph{Large-sample FID.}\label{app:fid50k}
The FID$_{50\mathrm{k}}$ column of Table~\ref{tab:imagenet} evaluates the MMD-setting students at epoch 300 with 50,000 generated images, 500 per class for classes 0--99, sampled with the native $K=3$ grid, against the Inception statistics of all 129,395 real training images of these classes. Training and evaluation are restricted to classes 0--99 as in Miao et al.~\citep{miao2024diverse}: the rewards are computed on these classes and the teachers are post-trained on them, so the other 900 classes carry no teacher adaptation to transfer.

\paragraph{Exact consistency gradient.}\label{app:exact-jvp}
With the consistency residual differentiated through the directional derivative as well, and everything else as in the $\lambda_c=0.01$ run, training diverges: FID is 31.3 at epoch 25 and 82.3 at epoch 300, against 18.23 for the stop-gradient update.

\paragraph{State source.}\label{app:state-source}
We compare the native rollout with the state acquisition of distribution matching distillation~\citep{yin2024dmd}: the student generates samples with its native sampler, and re-noised copies of these samples are the queried states. The loss (instantaneous velocity with $\lambda_c=0.01$), the MMD teacher, the initialization, the number of queries and updates, and the evaluation are unchanged. The student trained on native rollout states reaches FID 18.23, the one trained on re-noised samples 22.56.

\paragraph{What the map target carries.}\label{app:map-target}
Let $\psi^{T_0}$ and $\psi^{S_0}$ denote the pretrained maps from which the teacher and the student start. The teacher map splits as
$\psi^{T}=\psi^{S_0}+(\psi^{T}-\psi^{T_0})+(\psi^{T_0}-\psi^{S_0})$:
the student's own map, the teacher's reward adaptation, and the difference between the two pretrained models.
We regress the deterministic map onto $\psi^{S_0}+(\psi^{T}-\psi^{T_0})+\omega(\psi^{T_0}-\psi^{S_0})$, where $\omega=1$ is map regression itself, with every other setting unchanged (MMD teacher, epoch 100).
Table~\ref{tab:map-target} shows a monotone response: precision falls from 0.800 to 0.538 as the pretrained difference enters the target.
Restricting queries to native intervals (0.541) or replacing the target by the integrated teacher velocity (0.533) leaves precision unchanged, so the intervals and the teacher's map are not the cause.
Capacity, optimizer, and budget are the same for every $\omega$, so the response is to the target. The adapter rank is not the limit: with rank 256 instead of 64, deterministic map regression ends at FID 32.48 against 32.84, still far above the base's 22.68, and instantaneous velocity with $\lambda_c=0.01$ at 18.42 against 18.23 (MMD teacher, epoch 300).
Along the student's rollouts the pretrained difference is twice as large as the adaptation (RMS 0.308 against 0.148 of the teacher's displacement), so map regression mostly fits what separates the two pretrained models rather than what the reward taught the teacher. The velocity target carries the same difference in the same proportion (0.273 against 0.117), and there the student learns it: integrated with $K=15$, the velocity-only student reaches FID 19.17 against 22.87 for the base, and a consistency weight of 0.01 carries this to the $K=3$ map (18.14; Table~\ref{tab:step-count}). The pretrained difference is thus not absorbed as a long-range regression target and is learned as a local one. Heavier consistency weights hurt through the velocity term, whose error they raise (Supp.~\ref{app:consistency-diagnostics}): FID 21.05 and 43.07 for $\lambda_c=0.1$ and 1 (Table~\ref{tab:imagenet-step-comparison}).

\begin{table}[!htbp]\centering\footnotesize
\setlength{\tabcolsep}{6pt}
\begin{tabular}{@{}lcccc@{}}
\toprule
$\omega$ & 0 & 0.25 & 0.5 & 1 (map regression)\\
\midrule
Precision $\uparrow$ & \textbf{0.800} & 0.767 & 0.727 & 0.538\\
\bottomrule
\end{tabular}
\caption{\textbf{The more of the pretrained difference enters the target, the lower the precision.} Deterministic map regression with a fraction $\omega$ of the pretrained-generator difference in its target (MMD teacher, epoch 100, $K=3$).}
\label{tab:map-target}
\end{table}

\begin{table}[!htbp]
\centering\footnotesize
\setlength{\tabcolsep}{4pt}
\begin{tabular}{@{}lrrrr@{}}
\toprule
$K$ & 3 & 4 & 7 & 15\\
\midrule
B/2 base & 22.48 & 22.72 & 22.88 & 22.87\\
Deterministic map & 36.92 & 32.31 & 28.35 & 27.53\\
Source induced & \underline{20.42} & \underline{19.99} & 19.77 & 19.71\\
Velocity, $\lambda_c=0$ & 23.58 & 21.44 & \underline{19.29} & \underline{19.17}\\
\ours Velocity, $\lambda_c=0.01$ (ours) & \textbf{18.14} & \textbf{18.33} & \textbf{18.33} & \textbf{18.56}\\
\bottomrule
\end{tabular}
\caption{\textbf{With consistency, FID is stable across step counts.} The same epoch-300 checkpoints (MMD teacher) evaluated with different $K$ on the time grid of FlowMap-GRPO~\citep{li2026flowmapgrpo}, on which the base differs slightly from Table~\ref{tab:imagenet}; FID $\downarrow$; best in bold, second underlined.}
\label{tab:step-count}
\end{table}

Table~\ref{tab:step-sweep-all} lists the step sweep for all configurations.

\paragraph{Consistency-model distillation.}
We additionally apply FlowMap-OPD to a fixed-endpoint CM parameterization,
distilling an MMD-adapted XL/2 teacher into a B/2 student.
Table~\ref{tab:cm-transfer} compares endpoint-map and source-induced velocity
supervision at epoch 150. Source-induced supervision achieves the lowest
FID among the compared recipes at both $K=3$ and $K=4$, supporting the use of
FlowMap-OPD with CM students as well as two-time flow-map students.

\begin{table}[!htbp]\footnotesize
\begin{minipage}[t]{0.482\textwidth}\centering
\caption{\textbf{All configurations across step counts.} Epoch-300 checkpoints on the time grid of FlowMap-GRPO~\citep{li2026flowmapgrpo}. Lower test FID is better; among students, best in bold, second underlined.}
\label{tab:step-sweep-all}
\setlength{\tabcolsep}{2.5pt}
\begin{tabular}{@{}lrrrrr@{}}\toprule
Supervision & Epoch & $K=3$ & $K=4$ & $K=7$ & $K=15$\\\midrule
B/2 base & 0 & 22.48 & 22.72 & 22.88 & 22.87\\\midrule
Velocity, $\lambda_c=0$ & 300 & 23.58 & 21.44 & \underline{19.29} & \underline{19.17}\\
\ours Velocity, $\lambda_c=0.01$ (ours) & 300 & \textbf{18.14} & \textbf{18.33} & \textbf{18.33} & \textbf{18.56}\\
Velocity, $\lambda_c=0.1$ & 300 & 21.59 & 21.01 & 20.42 & 20.41\\
Velocity, $\lambda_c=1$ & 300 & 41.40 & 40.33 & 37.19 & 34.74\\
Velocity, $\lambda_c=10$ & 300 & 395.48 & 392.19 & 388.29 & 386.52\\
Stochastic map & 300 & 31.40 & 27.52 & 24.63 & 23.88\\
Deterministic map & 300 & 36.92 & 32.31 & 28.35 & 27.53\\
Source-induced velocity & 300 & \underline{20.42} & \underline{19.99} & 19.77 & 19.71\\
Destination-induced velocity & 300 & 46.25 & 41.92 & 36.13 & 33.65\\
\bottomrule\end{tabular}
\end{minipage}
\hfill
\begin{minipage}[t]{0.50\textwidth}\centering
\caption{\textbf{Velocity supervision also wins for consistency-model students.} CM distillation at epoch 150. We report FID at $K=3$ and
$K=4$ generation; lower is better.}
\label{tab:cm-transfer}
\setlength{\tabcolsep}{3pt}
\begin{tabular}{@{}>{\raggedright\arraybackslash\hangindent=1em}p{0.53\linewidth}rr@{}}
\toprule
Supervision & FID at $K=3$ $\downarrow$ & FID at $K=4$ $\downarrow$\\
\midrule
Stochastic endpoint map, $\lambda=0.1$ & 41.01 & 39.72\\
Stochastic endpoint map, $\lambda=0.2$ & 40.91 & 39.67\\
Stochastic endpoint map, $\lambda=0.4$ & 40.66 & 39.28\\
Stochastic endpoint map, $\lambda=0.8$ & 40.34 & 38.78\\
Stochastic endpoint map, $\lambda=1.0$ & \underline{40.12} & \underline{38.54}\\
Deterministic endpoint map & 40.96 & 39.68\\
Source-induced velocity & \textbf{22.45} & \textbf{21.03}\\
\bottomrule
\end{tabular}
\end{minipage}
\end{table}

\section{Additional Text-to-Image Results}\label{app:t2i-results}
This section details the baselines and extensions of Sections~\ref{sec:t2i} and~\ref{sec:separation-exp}, followed by the consistency analysis on text-to-image models.
\subsection{Multi-Step OPD Baseline}\label{app:opd-baseline}
The student is the original multi-step SD3.5-Medium with a fresh LoRA; the teachers are the public Flow-GRPO SD3.5-Medium specialists (Text, PickScore, GenEval), the teachers of the official Flow-OPD configuration. Flow-OPD rolls out ten SDE steps (noise 0.7) with classifier-free guidance 4.5 and matches student and teacher Gaussian transitions at every step. The training budget equals FlowMap-OPD's: 300 updates, 1{,}152 rollouts per round, the same prompt mix, LoRA, optimizer and EMA. We report $K=4$ sampling with guidance 4.5.

\needspace{10\baselineskip}
\subsection{Merging the Specialist LoRAs}\label{app:lora-merge}
\begin{wraptable}{r}{0.50\textwidth}
\vspace{-\intextsep}
\centering\footnotesize
\setlength{\tabcolsep}{4pt}\renewcommand{\arraystretch}{0.92}
\caption{\textbf{Merging the specialist LoRAs stays below distillation.} Merged specialist LoRAs against FlowMap-OPD; GenEval / OCR / PickScore, higher is better.}
\label{tab:lora-merge}
\begin{tabular}{@{}lcc@{}}
\toprule
Model & FLUX.1-lite-8B & SD3.5-Medium\\
\midrule
Base & 0.504 / 0.349 / 20.98 & 0.726 / 0.569 / 21.69\\
Merge (avg) & 0.742 / 0.657 / 22.12 & 0.789 / 0.755 / 22.38\\
Merge (sum) & 0.638 / 0.619 / 21.48 & 0.667 / 0.543 / 21.54\\
TIES & 0.823 / 0.652 / 21.99 & 0.816 / 0.779 / 22.49\\
Specialists (own task) & 0.845 / 0.850 / 23.08 & 0.868 / 0.912 / 23.52\\
\ours FlowMap-OPD (ours) & \textbf{0.858} / \textbf{0.883} / \textbf{23.15} & \textbf{0.872} / 0.905 / 23.40\\
\bottomrule
\end{tabular}
\end{wraptable}
The specialists are LoRAs on one base model, so they can be merged without training. Table~\ref{tab:lora-merge} evaluates three standard merges of the three adapters with the samplers and test sets of the main tables ($K=4$): the average of the three weight updates (Merge (avg)), their sum (Merge (sum)), and TIES~\citep{yadav2023ties} with density 0.2. Every merge stays below the specialists on every task, and FlowMap-OPD exceeds the best merge on all three.\par\WFclear

\subsection{Multi-Step Initialization}\label{app:multistep-init}
The student is the original SD3.5-Medium with a fresh LoRA and a second time input $r$ inserted as in AnyFlow: the time embedding is $0.75\,\mathrm{emb}(t)+0.25\,\mathrm{emb}_r(r)$, with $\mathrm{emb}_r$ initialized as a copy of the timestep embedder, so the network equals SD3.5-Medium at $r=t$ at initialization. Teachers, loss, $\lambda_c=0.001$, optimizer and budget are those of the SD3.5 FlowMap-OPD run; training rollouts use random $K=4$ time grids, evaluation the native $K=4$ grid.

\needspace{15\baselineskip}
\subsection{Consistency-Model Students}\label{app:scm-t2i}
FlowMap-OPD also applies to consistency models. We use the public sCM (TrigFlow) checkpoint of FLUX.1-lite-8B from T2I-Distill~\cite{pu2025fewstep} as the base of teachers and student. An sCM outputs only its endpoint $G^\theta(x,t)$, so the supervised quantity is the endpoint: the student collects states with its native rollout of Eq.~\eqref{eq:scmrollout} and, at each state, regresses the endpoint of the specialist of the prompt's task. The three specialists are sCM LoRAs post-trained on OCR, PickScore, and GenEval, and the remaining settings follow Section~\ref{sec:t2i}, with a learning-rate warmup of ten updates and 400 updates in total. Teachers and student share one base, so the endpoint target contains no pretrained-model difference, and one student recovers 91\%, 95\%, and 77\% of the specialists' gains over the base on OCR, PickScore, and GenEval (Table~\ref{tab:scm-t2i}).

\begin{table}[!htbp]\centering\footnotesize
\setlength{\tabcolsep}{6pt}
\begin{tabular}{@{}lccc@{}}
\toprule
sCM model & GenEval & OCR & PickScore\\
\midrule
Base & 0.451 & 0.490 & 20.35\\
Specialists (own task) & 0.805 & 0.886 & 23.48\\
\ours FlowMap-OPD student & 0.724 & 0.852 & 23.32\\
\bottomrule
\end{tabular}
\caption{\textbf{FlowMap-OPD also consolidates consistency-model specialists.} FLUX.1-lite-8B: one sCM student distilled from three sCM specialists with endpoint supervision, sampled with $K=4$.}\label{tab:scm-t2i} 
\end{table}

\subsection{Two-Task Consolidation versus Mixed-Reward GRPO}\label{app:two-task}
\begin{wrapfigure}{r}{0.38\textwidth}
\vspace{-\intextsep}
\centering
\includegraphics[width=\linewidth]{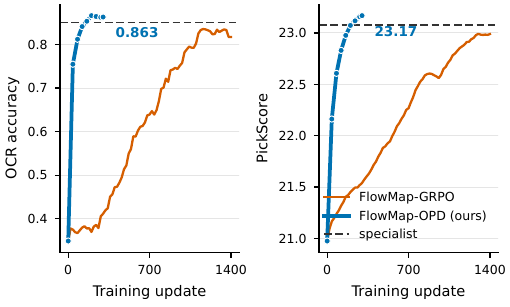}
\caption{\textbf{Distillation passes both specialists by update 300; reward training stays below them.} EMA OCR and PickScore of two-teacher FlowMap-OPD and FlowMap-GRPO with reward weights 0.5/0.5; dashed lines are the specialists.}
\label{fig:two-teacher-grpo}
\end{wrapfigure}
Given already available OCR and PickScore specialist teachers, we compare
OPD consolidation with directly optimizing both rewards using multi-reward
FlowMap-GRPO. Both settings
use a 50/50 mixture of OCR and PickScore training prompts, the same
pretrained MeanFlow backbone, and the same rollouts ($K=4$). GRPO uses training settings
following FlowMap-GRPO~\citep{li2026flowmapgrpo}, with OCR/PickScore
reward weights of 0.5/0.5. Both methods use eight H100 GPUs.

Figure~\ref{fig:two-teacher-grpo} compares the training progress of OPD and GRPO.
At update 300, OPD reaches OCR 0.863 and PickScore 23.17,
compared with 0.816 and 22.98 for GRPO at update 1,400. OPD already exceeds these GRPO scores at update 200,
with OCR 0.866 and PickScore 23.07.
GRPO has stopped improving by then: its OCR peaks at 0.835 near update 1,200 and its PickScore stays near 22.98 from update 1,300, both below the specialists (0.850 and 23.08), which the distilled student exceeds. Joint optimization of the two rewards saturates below the specialist level that distillation reaches.

\needspace{17\baselineskip}
\subsection{Text-to-Image Consistency Analysis}\label{app:consistency-diagnostics}
\begin{wrapfigure}{r}{0.30\textwidth}
\vspace{-\intextsep}
\centering
\includegraphics[width=\linewidth]{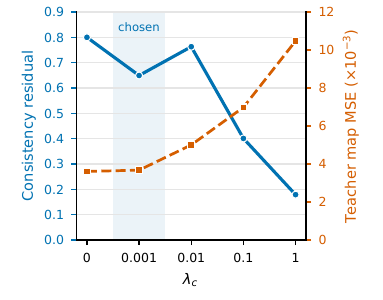}
\caption{\textbf{A large $\lambda_c$ lowers the consistency residual and raises the map error.} Source residual (solid, left axis) and error to the teacher map (dashed, right axis) at update 200.}
\label{fig:consistency-monitor}
\end{wrapfigure}
\paragraph{Consistency strength.}
We vary $\lambda_c$ and compare the students after 200 training steps.
Table~\ref{tab:consistency-errors} reports student consistency and
teacher-matching errors. Weights of $0.1$ and above reduce the consistency
residual, but large weights increase the difference from the teacher and
reduce task performance (Table~\ref{tab:consistency-t2i}). The coefficient
therefore controls how strongly the student fits the teacher relative to
its own consistency constraint.

\paragraph{Consistency strength and teacher transfer.}\label{sec:consistency-experiment}\label{sec:evidence}

We vary $\lambda_c\in\{0,0.001,0.01,0.1,1\}$ in velocity-supervised
FlowMap-OPD with the same three specialist teachers.
Table~\ref{tab:consistency-t2i} reports full-test-set scores at
update 200 from these separate sweep runs; the application above uses
update 300 with $\lambda_c=0.001$.

Figure~\ref{fig:consistency-monitor} compares consistency loss and teacher-map error at update 200.
Increasing $\lambda_c$ from 0 to 1 reduces the unweighted source
residual from 0.800 to 0.179, but increases teacher-map MSE from
0.0036 to 0.0105. All three task scores decline at the larger tested penalties.
The smaller coefficient $\lambda_c=0.001$ improves all three scores
over $\lambda_c=0$ (Table~\ref{tab:consistency-t2i}).
The weight is selected on training prompts: averaged over updates 160--200, the rollout rewards on the OCR and PickScore training prompts are 0.869 and 23.07 at $\lambda_c=0.001$, against 0.822 and 22.94 at 0, 0.823 and 22.69 at 0.01, 0.727 and 22.05 at 0.1, and 0.266 and 21.26 at 1; the test scores are computed after this choice. The margin over $\lambda_c=0$ widens with training: at update 280 the full-test OCR and PickScore are 0.886 and 23.14 at $\lambda_c=0.001$, against 0.833 and 22.99 at 0.

A small penalty improves all three scores, as on ImageNet; larger weights lower the residual further but move the student's map away from its teachers.

\begin{table}[!htbp]
\footnotesize
\begin{minipage}[t]{0.56\textwidth}
\centering
\caption{\textbf{Consistency and teacher matching trade off.} Errors after 200 training steps,
averaged over the preceding 20 steps at the last supervised interval.
Velocity errors are mean squared differences from the teacher;
map-composition error measures agreement between direct and composed student maps.}
\label{tab:consistency-errors}
\setlength{\tabcolsep}{5pt}
\begin{tabular}{@{}lrrrr@{}}
\toprule
$\lambda_c$ & \shortstack[r]{Source\\residual} & \shortstack[r]{Velocity\\mismatch} & \shortstack[r]{Average-velocity\\mismatch} & \shortstack[r]{Map composition\\error}\\
\midrule
0 & 0.800 & 0.017 & 0.031 & 0.003391\\
\ours 0.001 (ours) & 0.649 & 0.020 & 0.032 & 0.003148\\
0.01 & 0.763 & 0.053 & 0.043 & 0.003504\\
0.1 & 0.401 & 0.088 & 0.060 & 0.002558\\
1 & 0.179 & 0.171 & 0.090 & 0.001996\\
\bottomrule
\end{tabular}
\end{minipage}\hfill
\begin{minipage}[t]{0.40\textwidth}
\centering
\caption{\textbf{A small consistency weight helps, a large one hurts.} Three-specialist distillation at update 200 for each $\lambda_c$, full test sets; best in bold, second underlined.}\label{tab:consistency-t2i}
\setlength{\tabcolsep}{5pt}
\begin{tabular}{@{}lccc@{}}
\toprule
$\lambda_c$ & GenEval $\uparrow$ & OCR $\uparrow$ & PickScore $\uparrow$\\
\midrule
0 & \underline{0.827} & \underline{0.830} & \underline{22.85}\\
\ours 0.001 (ours) & \textbf{0.838} & \textbf{0.845} & \textbf{22.98}\\
0.01 & 0.725 & 0.806 & 22.65\\
0.1 & 0.568 & 0.646 & 22.03\\
1 & 0.450 & 0.340 & 21.59\\
\bottomrule
\end{tabular}
\end{minipage}
\end{table}

\providecommand{\gallerywidth}{0.68\textwidth}
\section{Paired Qualitative Comparisons}\label{app:paired-gallery}
We compare the base, all three specialist teachers and the same student.
Within each row, the models share the prompt, initial noise, few-step
sampling grid and guidance. Figures~\ref{fig:specialist-comparison}--\ref{fig:gallery-cross-3} mix the three tasks; Figures~\ref{fig:gallery-ocr-1}--\ref{fig:gallery-geneval-2} show each task in turn.

\subsection{Cross-task Specialist--Student Comparisons}
Each panel combines examples from the three tasks to show specialist
capabilities and their consolidation in one student.
\begin{figure}[H]
\centering
\includegraphics[width=\gallerywidth]{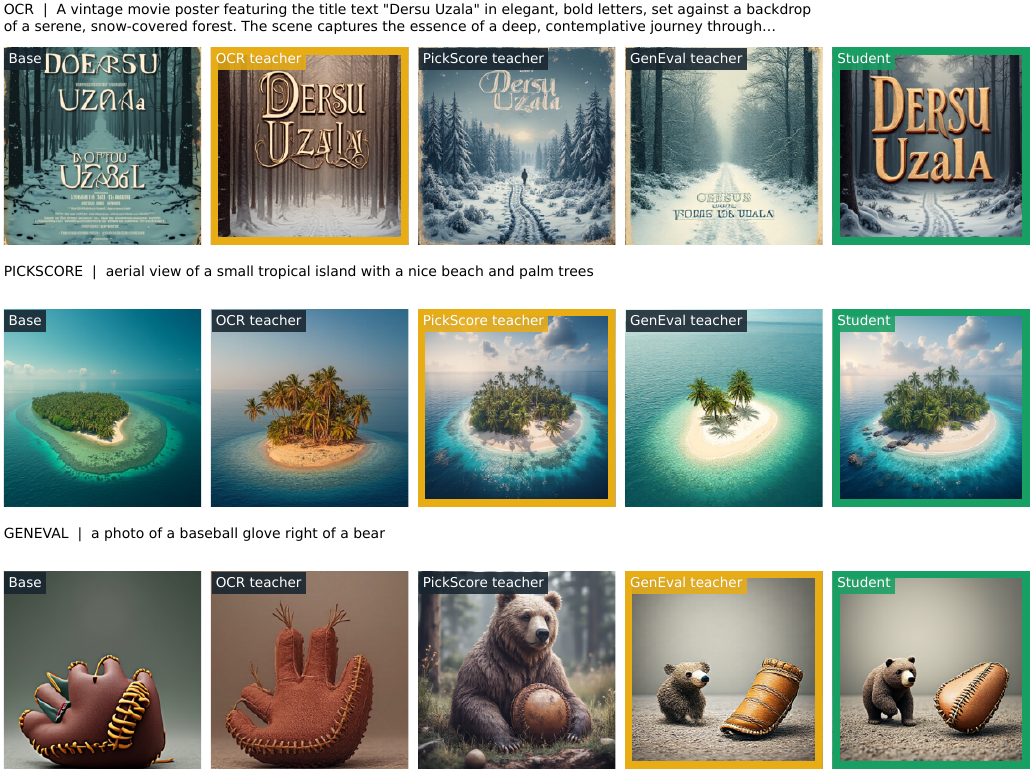}
\caption{\textbf{Complementary specialist capabilities in one student.}
Rows compare text rendering, visual preferences and object relations.
Columns show the base, OCR teacher, PickScore teacher, GenEval teacher
and the same student. Thick gold border: the specialist of the row's task; green border: the FlowMap-OPD student.}
\label{fig:specialist-comparison}
\end{figure}
\begin{figure}[H]
\centering
\includegraphics[width=\gallerywidth]{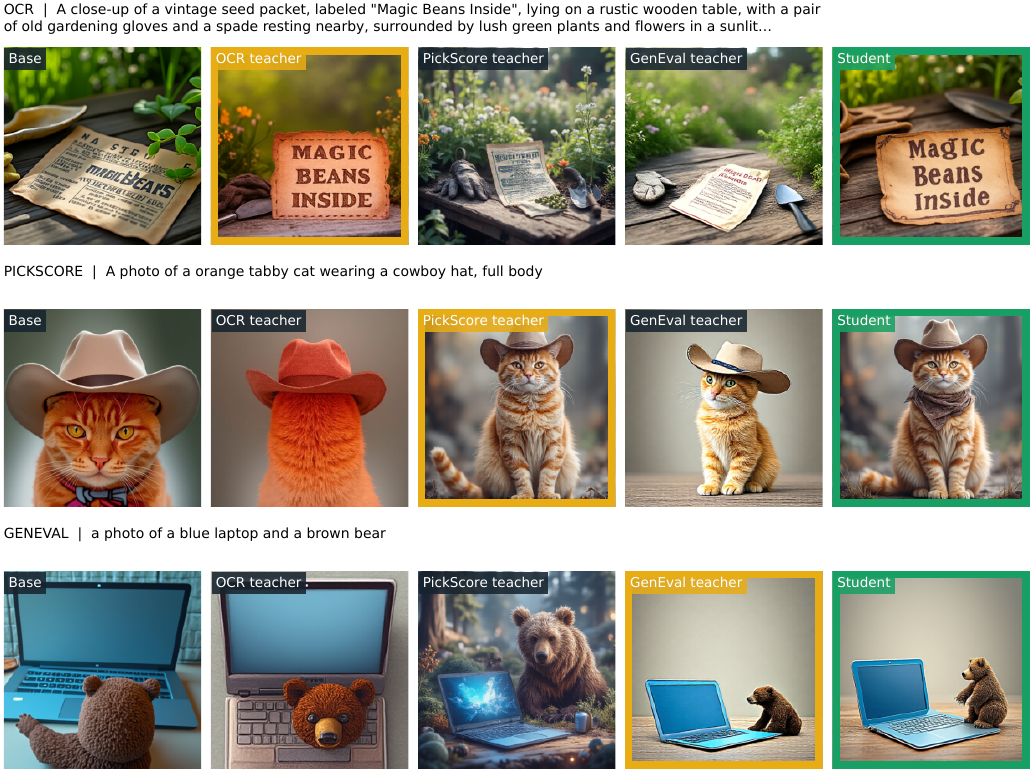}
\caption{\textbf{Cross-task specialist--student comparison.} Rows cover OCR, visual preferences and compositional generation. Each row compares the base, three specialist teachers and the same student under a shared prompt, initial noise and few-step sampler. Thick gold border: the specialist of the row's task; green border: the FlowMap-OPD student.}
\label{fig:gallery-cross-2}
\end{figure}

\begin{figure}[H]
\centering
\includegraphics[width=\gallerywidth]{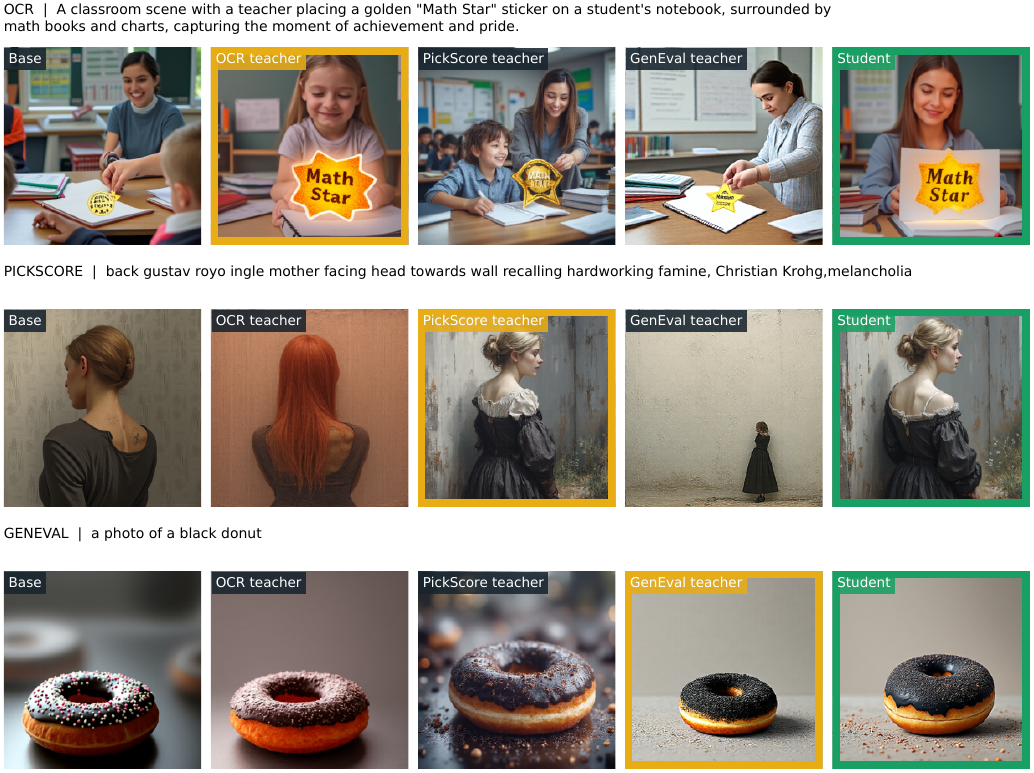}
\caption{\textbf{Cross-task specialist--student comparison.} Rows cover OCR, visual preferences and compositional generation. Each row compares the base, three specialist teachers and the same student under a shared prompt, initial noise and few-step sampler. Thick gold border: the specialist of the row's task; green border: the FlowMap-OPD student.}
\label{fig:gallery-cross-3}
\end{figure}

\subsection{OCR Capability Comparisons}

\begin{figure}[H]
\centering
\includegraphics[width=\gallerywidth]{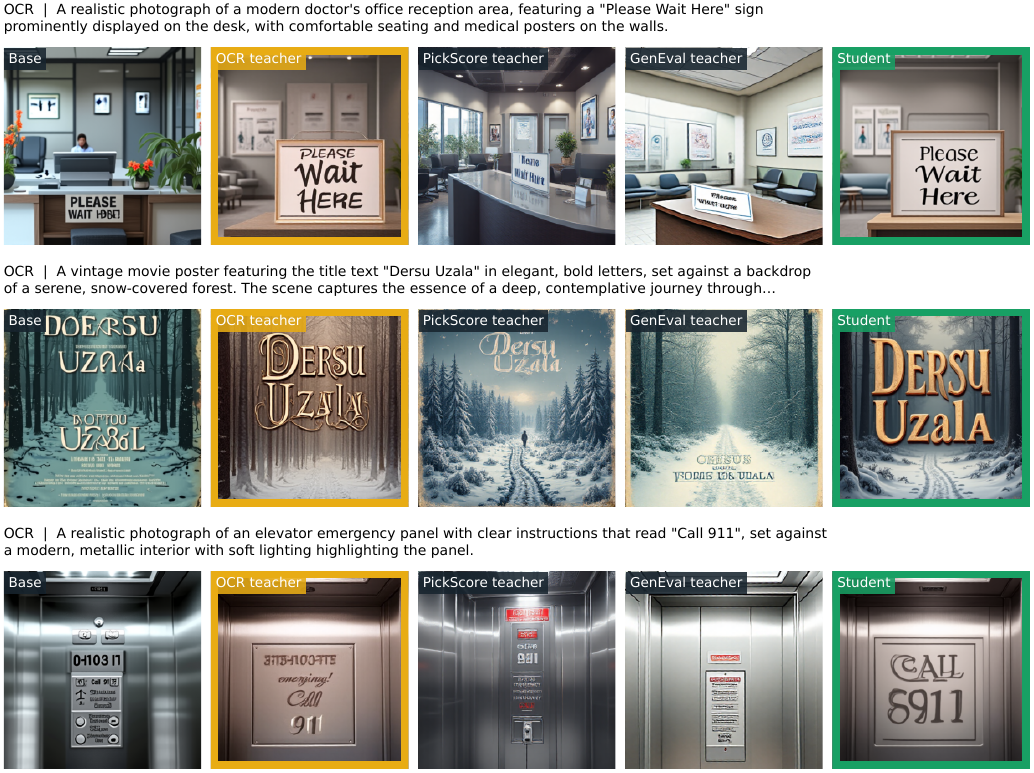}
\caption{\textbf{Paired OCR examples.} Columns show the base, OCR teacher, PickScore teacher, GenEval teacher, and the same student. Each row shares the prompt, initial noise and few-step sampler. Thick gold border: the specialist of the row's task; green border: the FlowMap-OPD student.}
\label{fig:gallery-ocr-1}
\end{figure}

\begin{figure}[H]
\centering
\includegraphics[width=\gallerywidth]{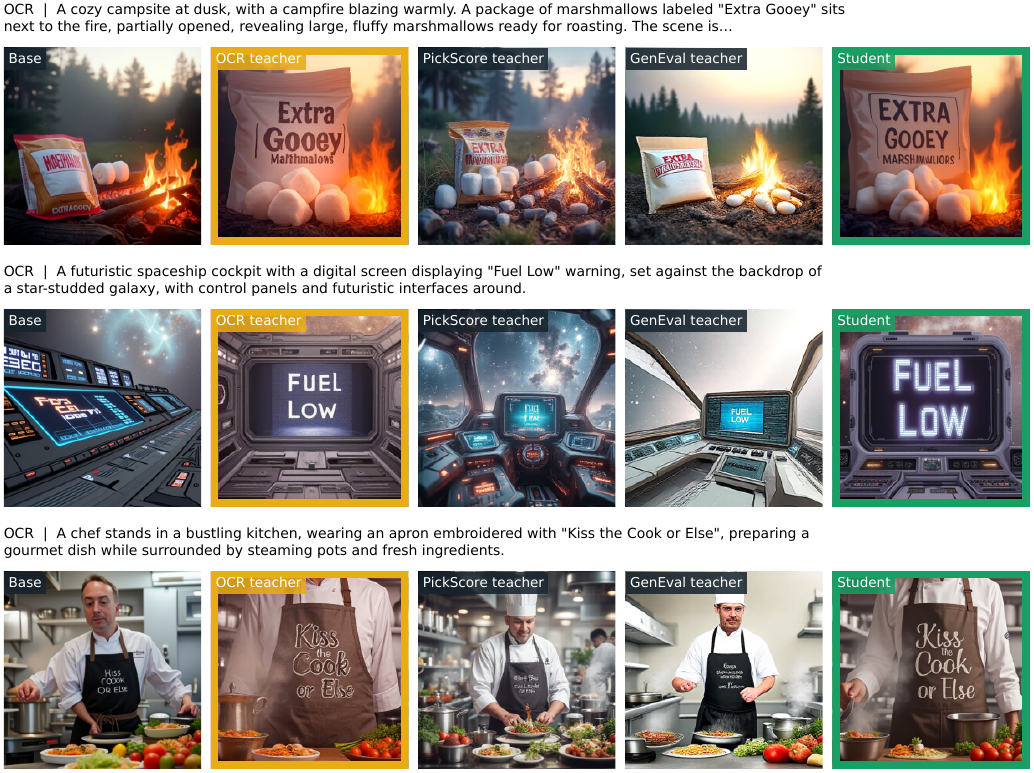}
\caption{\textbf{Paired OCR examples.} Columns show the base, OCR teacher, PickScore teacher, GenEval teacher, and the same student. Each row shares the prompt, initial noise and few-step sampler. Thick gold border: the specialist of the row's task; green border: the FlowMap-OPD student.}
\label{fig:gallery-ocr-2}
\end{figure}

\subsection{Visual Preference Comparisons}

\begin{figure}[H]
\centering
\includegraphics[width=\gallerywidth]{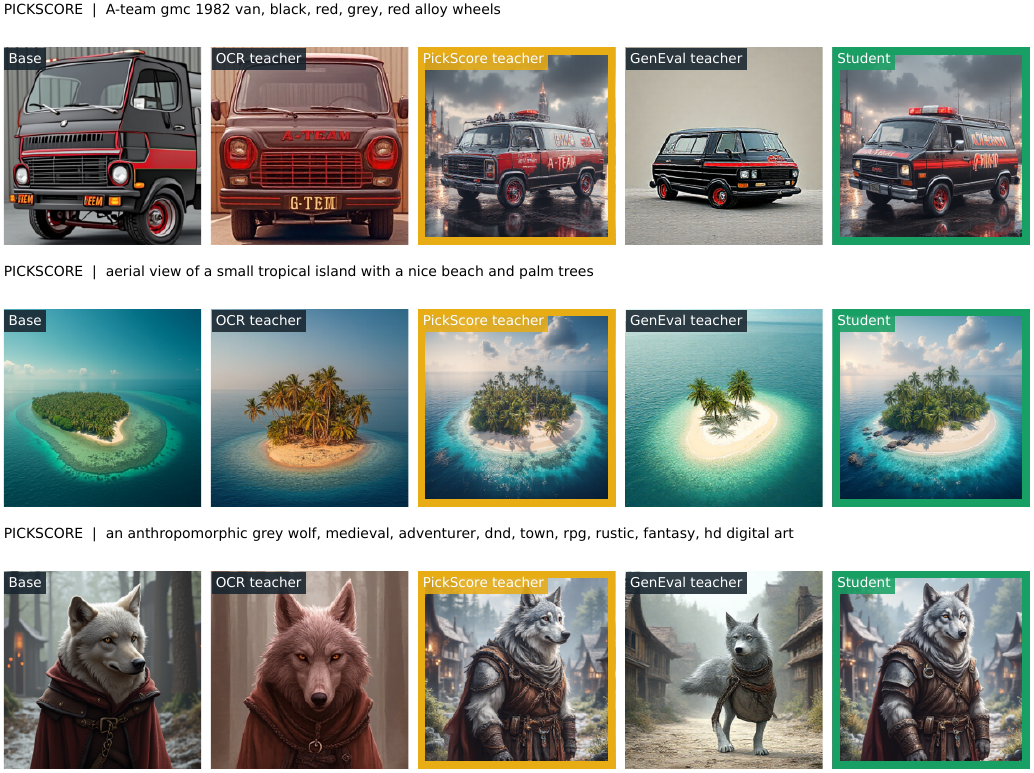}
\caption{\textbf{Paired PickScore examples.} Columns show the base, OCR teacher, PickScore teacher, GenEval teacher, and the same student. Each row shares the prompt, initial noise and few-step sampler. Thick gold border: the specialist of the row's task; green border: the FlowMap-OPD student.}
\label{fig:gallery-pick-1}
\end{figure}

\begin{figure}[H]
\centering
\includegraphics[width=\gallerywidth]{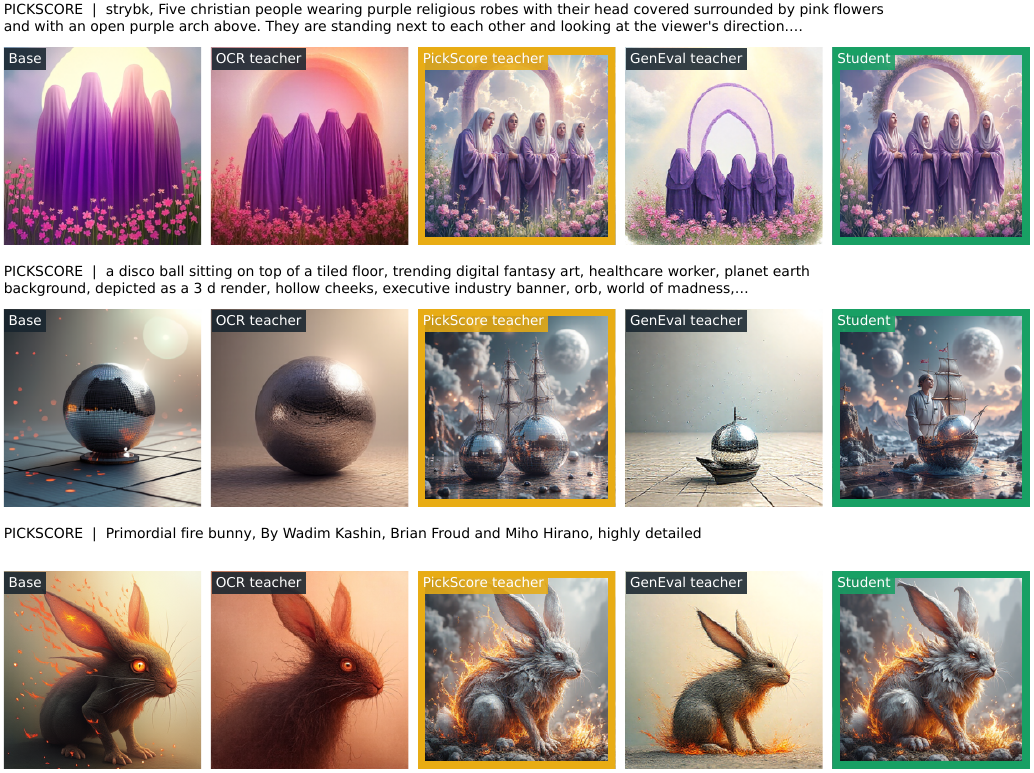}
\caption{\textbf{Paired PickScore examples.} Columns show the base, OCR teacher, PickScore teacher, GenEval teacher, and the same student. Each row shares the prompt, initial noise and few-step sampler. Thick gold border: the specialist of the row's task; green border: the FlowMap-OPD student.}
\label{fig:gallery-pick-2}
\end{figure}

\subsection{Compositional Generation Comparisons}

\begin{figure}[H]
\centering
\includegraphics[width=\gallerywidth]{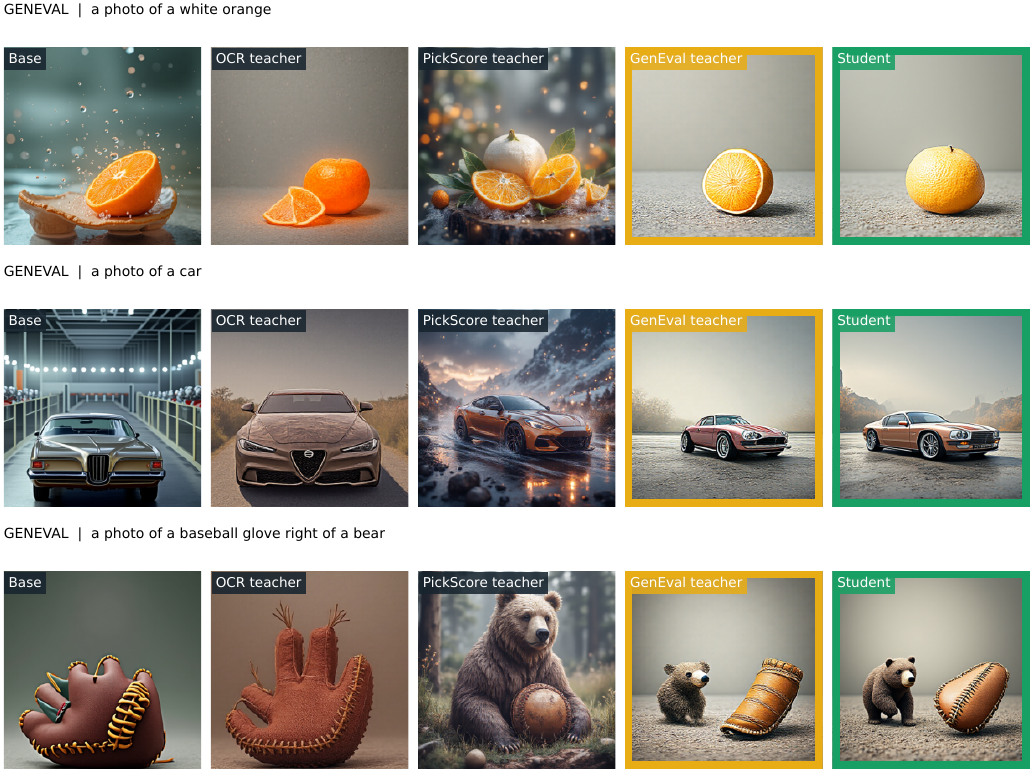}
\caption{\textbf{Paired GenEval examples.} Columns show the base, OCR teacher, PickScore teacher, GenEval teacher, and the same student. Each row shares the prompt, initial noise and few-step sampler. Thick gold border: the specialist of the row's task; green border: the FlowMap-OPD student.}
\label{fig:gallery-geneval-1}
\end{figure}

\begin{figure}[H]
\centering
\includegraphics[width=\gallerywidth]{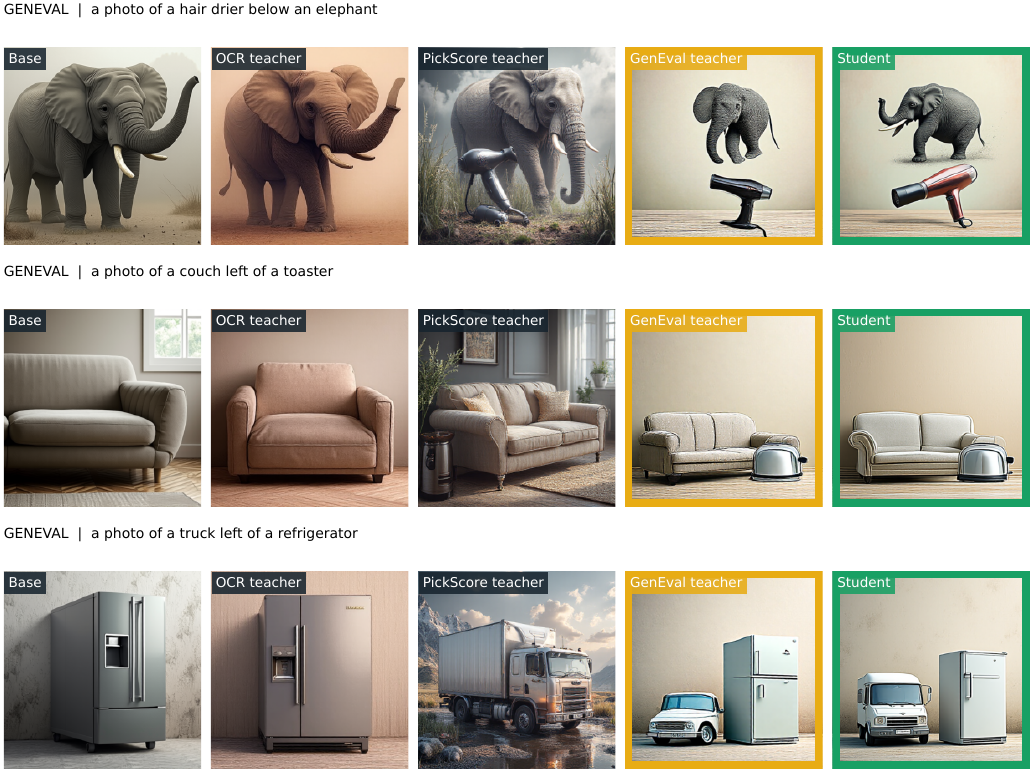}
\caption{\textbf{Paired GenEval examples.} Columns show the base, OCR teacher, PickScore teacher, GenEval teacher, and the same student. Each row shares the prompt, initial noise and few-step sampler. Thick gold border: the specialist of the row's task; green border: the FlowMap-OPD student.}
\label{fig:gallery-geneval-2}
\end{figure}

\flushbottom
\section{Rollout Separation and Sampling}
\label{app:occupancy}
This section proves the statements of Section~\ref{sec:rk-theory}: the marginal form of the objective (Eq.~\eqref{eq:rk-original-marginals}), the effect of replacing the rollout, and the time marginal of rollouts on random grids.

\subsection{Sampled states, comparison times, and rollout distributions}
\label{app:general-queries}

We use the state marginals and local comparison from
Section~\ref{sec:rk-theory}. A query $z$ retains the sampled state
and the times needed by that comparison.

Let $\mathcal X=\mathbb R^d$ be the state space. Sample a grid
$\Gcal=(1=t_0>\cdots>t_N=0)\sim\gamma$, where $\gamma$ is the
chosen distribution of time grids and $N$ is the number of steps.
A fixed grid corresponds to a point mass. Conditioning on a task or prompt
is retained whenever it affects the comparison.

For rollout policy $\Tcal$ and student parameters $\theta$, write
$\Tcal^\theta(d\tau\mid \Gcal)$ for the conditional trajectory distribution, where
$\tau=(x_0,\ldots,x_{N})$. Here $x_i:=x_{t_i}$ denotes the state at the corresponding grid time. For a measurable set $A\subseteq\mathcal X$, its state marginals are
\begin{equation}
    d_i^{\Tcal,\theta}(A\mid \Gcal)
    =\Tcal^\theta(x_i\in A\mid \Gcal).
    \label{eq:app-state-marginal}
\end{equation}
The rollout may be deterministic or history-dependent; no state density is required.
Throughout this appendix, we fix a local comparison and its comparison-variable
sampling rule. The collected query distribution is held fixed when
differentiating the local objective.

Choose measurable weights $a_i(\Gcal)\geq0$ with
$\sum_{i=0}^{N-1}a_i(\Gcal)=1$. Given $\Gcal$, the query index is selected
independently of the realized path using these weights. Let $\zeta$ collect any additional comparison choices, such as a
destination time or an anchor interval, and let
$q(d\zeta\mid x,i,\Gcal)$ be their conditional sampling distribution.
For the chosen local loss $\ell$, define its conditional average
\begin{equation}
    \overline\ell^\theta(x,i,\Gcal)
    =\int\ell(\theta,T;(\Gcal,i,x),\zeta)
                    q(d\zeta\mid x,i,\Gcal).
    \label{eq:app-integrated-loss}
\end{equation}
Here $z=(\Gcal,i,x)$. We assume the expected absolute loss is finite.
For nonnegative losses, the marginal identity also holds with infinite
values by Tonelli's theorem.

\paragraph{Sampled states and separated comparisons.}
\label{app:separated-queries}
The population objective corresponding to the sampled update in
\eqref{eq:organizing-objective} is
\begin{equation}
 \mathcal L_{\Tcal}(\theta)
 =\mathbb E_{z\sim\rho^{\Tcal,\theta}}
       [\ell(\theta,T;z)],
 \label{eq:rk-separated-objective}
\end{equation}
where $\rho^{\Tcal,\theta}$ is the distribution of queries collected by
the frozen rollout, and $z$ retains the state and all times required by
the comparison. The shorthand $\ell(\theta,T;x_t,t)$ in the main text
averages any additional comparison choices at fixed $(x_t,t)$.
Specifically, draw $\Gcal\sim\gamma$, a rollout
$\tau\sim \Tcal^\theta(\cdot\mid\Gcal)$, and an index $I$
with probabilities $a_i(\Gcal)$. The resulting query is
$z=(\Gcal,I,x_{t_I})$.
The local comparison may be a conditional kernel KL, its regression
form, or a loss with a separately specified consistency penalty.
In Eq.~\eqref{eq:rk-separated-objective}, auxiliary variables are averaged
as $\ell(\theta,T;z)=\int\ell(\theta,T;z,\zeta)q(d\zeta\mid z)$.
For uniform index weights $a_i=1/N$, conditioning on the grid gives
\begin{equation}
 \mathcal L_{\Tcal}^{\Gcal}(\theta)
 =\frac1N\sum_{i=0}^{N-1}
   \mathbb E_{\substack{x\sim d_i^{\Tcal,\theta}(\cdot\mid \Gcal)\\
          \zeta\sim q(\cdot\mid \Gcal,i,x)}}
   [\ell(\theta,T;z_i,\zeta)].
 \label{eq:rk-marginal-objective}
\end{equation}
Here $z_i=(\Gcal,i,x)$. Averaging over $\Gcal\sim\gamma$ gives
\eqref{eq:rk-separated-objective}. If the comparison and its auxiliary
rule depend on the grid only through the source and destination times,
we use $z=(x,t,r)$, where $r=t_{i+1}$ and $t=t_i$. Residuals and induced predictions use the same argument order.
For continuous flow-map marginals, let $\varsigma$ denote the distribution
of comparison times. The corresponding expression is
\begin{equation}
 \mathcal L_{\psi}(\theta)
 =\int_{[0,1]}\int
 \overline\ell(\theta,T;x,t)\,
 (\psi^{\theta}_{1\to t})_\#p_1(dx)\,\varsigma(dt),
 \label{eq:rk-separated-pushforward}
\end{equation}
where $\overline\ell(\theta,T;x,t)=
\mathbb E_{\zeta\sim q(\cdot\mid x,t)}[\ell(\theta,T;x,t,\zeta)]$.

\subsection{Marginal formulation of the objective}
\label{app:marginal-proof}

Define the base-query distribution on $(\Gcal,i,x)$ by
\begin{equation}
    \rho^{\Tcal,\theta}(d\Gcal,\{i\},dx)
      =\gamma(d\Gcal)a_i(\Gcal)d_i^{\Tcal,\theta}(dx\mid \Gcal),
    \label{eq:app-occupancy-measure}
\end{equation}
and the augmented query distribution by
\begin{equation}
    \varrho^{\Tcal,\theta}(d\Gcal,\{i\},dx,d\zeta)
      =\rho^{\Tcal,\theta}(d\Gcal,\{i\},dx)
         q(d\zeta\mid x,i,\Gcal).
    \label{eq:app-augmented-occupancy}
\end{equation}
Both are probability measures.

\begin{appendixproposition}[General marginal sufficiency]
\label{prop:app-marginal}
Under the preceding measurability and integrability assumptions,
\begin{align}
 \mathcal L_{\Tcal}(\theta)
 &=\int\gamma(d\Gcal)\int \Tcal^\theta(d\tau\mid \Gcal)
       \sum_i a_i(\Gcal)\overline\ell^\theta(x_i,i,\Gcal) \notag\\
 &=\int\gamma(d\Gcal)\sum_i a_i(\Gcal)
       \int \overline\ell^\theta(x,i,\Gcal)
                    d_i^{\Tcal,\theta}(dx\mid \Gcal) \notag\\
 &=\int\ell(\theta,T;(\Gcal,i,x),\zeta)
                    \varrho^{\Tcal,\theta}(d\Gcal,\{i\},dx,d\zeta).
 \label{eq:app-marginal-identity}
\end{align}
\end{appendixproposition}
\begin{proof}
Condition on $\Gcal$. Exchange the finite sum and the path integral.
Each summand is a function only of the coordinate $x_i$ and its retained
query variables, so integrating against the path distribution is equivalent to
integrating against its coordinate marginal. Integrate over $\Gcal$ and
then substitute \eqref{eq:app-integrated-loss}.
\end{proof}

For a fixed grid and $a_i=1/N$, the middle line is
\eqref{eq:rk-marginal-objective}. Sampling $I$ with probabilities $a_i(\Gcal)$
yields \eqref{eq:rk-separated-objective}. This gives the sampling form used in the main text.

\paragraph{Continuous-time form.}
\label{app:continuous-queries}
For the SDE kernels of the main text the local loss is $h_i f(t_i,x)$ with $f(t,x)=\|b_t^\theta(x)-b_t^T(x)\|^2/(2g_t^2)$, so the fixed-grid objective is the Riemann sum $\sum_i h_i\,\mathbb E_{x\sim d_i^{\Tcal,\theta}}[f(t_i,x)]$. If $f$ is bounded and continuous and the grid marginals converge weakly to the student-velocity marginals $d_t^{v,\theta}=(\psi^{v^\theta}_{1\to t})_\#p_1$ as the mesh vanishes, the sum converges to $\int_0^1\mathbb E_{x\sim d_t^{v,\theta}}[f(t,x)]\,dt$, which is Eq.~\eqref{eq:rk-time-integral}.

\subsection{Rollout replacement and mismatch}
\label{app:replacement}

\begin{appendixproposition}[General marginal-preserving rollout replacement]
\label{thm:app-replacement}
Fix $\theta$, and let $\Tcal$ and $\widetilde\Tcal$ satisfy the setup
of Appendix~\ref{app:general-queries}. Fix the same grid distribution $\gamma$,
weights $a_i$, auxiliary rule $q$, and local loss, integrable under both
query distributions. If
\[
 d_i^{\Tcal,\theta}(\cdot\mid \Gcal)
 =d_i^{\widetilde\Tcal,\theta}(\cdot\mid \Gcal)
\]
for $\gamma$-almost every $\Gcal$ and every index with $a_i(\Gcal)>0$, then
\[
 \mathcal L_{\Tcal}(\theta)
 =\mathcal L_{\widetilde\Tcal}(\theta).
\]
Equality of the base-query distributions is also sufficient. If the local
loss is differentiable in a neighborhood of $\theta$, with derivatives
dominated by integrable envelopes under both fixed augmented query distributions,
then their expected gradients with the query distribution held fixed agree as well. Neither
conclusion requires equality of the two path distributions.
\end{appendixproposition}
\begin{proof}
Proposition~\ref{prop:app-marginal} writes both objectives as integrals of
the same local loss against their augmented query distributions. Matching the
conditional marginals with the same $\gamma,a_i,q$ makes these distributions
equal. Under the derivative assumptions, apply dominated differentiation
to the same integral representation.
\end{proof}
Matching query distributions preserves the expected local loss and update;
cross-time correlation can still affect estimator variance.
Pairwise losses require pairwise query distributions, and batch-dependent
normalization must be included in the comparison.

\begin{appendixproposition}[General Wasserstein mismatch bound]
\label{prop:app-wasserstein-mismatch}
\label{prop:rk-mismatch}
Suppose $\overline\ell^\theta(\cdot,i,\Gcal)$ is $L_{i,\Gcal}$-Lipschitz on the comparison
state space and the two conditional state distributions have finite first moments.
For fixed weights, grid distribution, and auxiliary rule,
\begin{equation}
    |\mathcal L_{\Tcal}-\mathcal L_{\widetilde\Tcal}|
    \leq\int\gamma(d\Gcal)\sum_i a_i(\Gcal)L_{i,\Gcal}
       W_1\!\left(d_i^{\Tcal,\theta}(\cdot\mid \Gcal),
                  d_i^{\widetilde\Tcal,\theta}(\cdot\mid \Gcal)\right),
    \label{eq:app-wasserstein-mismatch}
\end{equation}
provided the right-hand side is integrable.
\end{appendixproposition}
\begin{proof}
Write $g(x)=\overline\ell^\theta(x,i,\Gcal)$ at fixed $(i,\Gcal)$.
For any coupling $\kappa_{i,\Gcal}$ of the two conditional state distributions,
\begin{align}
 \left|\int g(x)\,d_i^{\Tcal,\theta}(dx\mid\Gcal)
       -\int g(x)\,d_i^{\widetilde\Tcal,\theta}(dx\mid\Gcal)\right|
 &\leq\int|g(x)-g(y)|\,\kappa_{i,\Gcal}(dx,dy) \notag\\
 &\leq L_{i,\Gcal}\int\|x-y\|\,\kappa_{i,\Gcal}(dx,dy).
\end{align}
Taking the infimum over couplings yields $L_{i,\Gcal}W_1$ at each $(\Gcal,i)$.
Weight by $a_i(\Gcal)$ and integrate over the grid.
\end{proof}

For a fixed grid and uniform weights, every joint sampling of
$(x_i,\widetilde x_i)$ with the two required marginals yields
\begin{equation}
 |\mathcal L_{\Tcal}^\Gcal-\mathcal L_{\widetilde\Tcal}^\Gcal|
 \leq \frac1N\sum_i L_{i,\Gcal}
                 \mathbb E[\|x_i-\widetilde x_i\|\mid \Gcal].
 \label{eq:rk-paired-mismatch}
\end{equation}
The Lipschitz constant is that of the integrated local loss $\overline\ell^\theta$.

\subsection{Random-time training and unbiased estimation}
\label{app:random-grids}\label{app:source-time-law}\label{app:five-step-sampling}
Random-time training draws the interior times of each rollout at random. Take four independent uniform
interior times, drawn independently of $x_{t_0}\sim p_1$ and sorted as
$1>t_1>\cdots>t_4>0$, with $t_0=1$ and $t_5=0$.
For an integrable function $f$, selecting one of the five source indices
uniformly gives
\begin{equation}
 \mathbb E\!\left[\frac15\sum_{i=0}^4 f(t_i)\right]
 =\frac15f(1)+\frac45\int_0^1f(t)\,dt.
 \label{eq:app-source-time-law}
\end{equation}
Sorting leaves the sum over the four interior draws unchanged, proving
the identity. The selected time therefore follows
$\frac15\delta_1+\frac45\mathrm{Uniform}(0,1)$; individual order
statistics do not. Its density on $(0,1)$ is at least $4/5$, while the deployment segments of \eqref{eq:app-nu-v} have Lebesgue time marginal, so the time marginal of the ratio in \eqref{eq:app-coverage-ratios} is bounded by $5/4$. Uniform continuous-time weighting is recovered by
averaging only the four interior queries, or by assigning weights $5/4$
to interior terms and zero to the initial term in the five-query average.
A minimum-spacing rejection rule changes this distribution. Adjacent
interior times cover $0<r<t<1$, with their joint distribution determined
by the sorting rule.

Averaging all starts and sampling one uniformly have the same expected
loss but different cost and variance. More generally, sampling index $i$
with probability $\varpi_i>0$ to estimate weights $a_i$ requires the factor
$a_i/\varpi_i$. For independent rollouts at fixed parameters, integrable
weighted losses and state-fixed gradients obey the strong law of large
numbers. Within-rollout query correlations affect variance, not this
consistency of the empirical average.

\section{Flow--Velocity Consistency and Error Bounds}
\label{app:flow-proofs}

This section proves Theorem~\ref{thm:main-transfer} and its extensions: the consistency identities, the source and destination bounds, their propagation through a native rollout, and the population losses. We use the consistency residuals in
Eq.~\eqref{eq:fm-csrc} and the consistent teacher
$\psi^T$ defined in the main text.

\subsection{Flow--Velocity Consistency}
\label{app:exact-transfer}

\begin{appendixassumption}[Transport regularity]
\label{ass:fm-regularity}
The velocity fields are continuous in time and locally Lipschitz in state,
uniformly on compact subsets, with unique nonexplosive flows on the
intervals considered. The queried maps are continuously differentiable
in state and both times for $r<t$, extend continuously to the identity
boundary, and have the derivatives appearing in the residuals. All
required flow segments and arrival curves remain in the comparison
domains. Residuals extend continuously where boundary limits are used.
\end{appendixassumption}

\begin{appendixtheorem}[Flow-map identification from consistency]
\label{thm:fm-exact}
Let $v$ be a velocity field and $\psi^v_{t\to r}$ its flow map. Let
$\psi_{t\to r}$ be another family of maps with $\psi_{t\to t}(x)=x$.
Under Assumption~\ref{ass:fm-regularity}, if $\psi$ and $v$ satisfy
either the source or the destination relation in
\eqref{eq:fm-consistency-relations} throughout the comparison domain, then
\begin{equation}
 \psi_{t\to r}(x)=\psi^v_{t\to r}(x).
 \label{eq:fm-identification}
\end{equation}
\end{appendixtheorem}
For a consistent teacher, whose map $\psi^T$ is the flow of $v^T$, a student
with $c^{\theta,\mathrm{src}}=0$ or $c^{\theta,\mathrm{dst}}=0$ and
$v^\theta=v^T$ therefore satisfies
\begin{equation}
 \psi^\theta_{t\to r}=\psi^T_{t\to r},
 \label{eq:fm-teacher-identification}
\end{equation}
and composing the student maps recovers the teacher transport on any valid
sampling grid. Figure~\ref{fig:tworoutes} illustrates the two identities.

\begin{figure}[h]
\centering
\begin{minipage}[t]{0.49\linewidth}
\centering
\begin{tikzpicture}[x=1cm,y=1cm,font=\scriptsize,
  line cap=round,line join=round,>=stealth]
\path[use as bounding box] (0,-0.14) rectangle (3.85,2.78);
\coordinate (a) at (0.25,0.35);
\coordinate (c) at (0.90,0.42);
\coordinate (b) at (3.45,2.00);
\draw[black!35,line width=1.1pt]
  (a) .. controls (1.80,0.35) and (1.90,2.00) .. (b);
\draw[->,black!65,line width=1.1pt,dash pattern=on 3pt off 2pt]
  (a) .. controls (0.12,2.42) and (2.64,2.72) .. (b);
\node at (1.76,2.56) {$\psi_{t\to r}$};
\draw[->,black!65,line width=1.1pt,dash pattern=on 3pt off 2pt]
  (c) .. controls (2.27,0.14) and (3.67,0.77) .. (b);
\node at (2.64,0.16) {$\psi_{t-\iota\to r}$};
\draw[->,black,line width=1.4pt] (a)--(c);
\node at (0.65,0.03) {$-\iota v_t$};
\node[above] at (0.78,0.54) {$x_{t-\iota}$};
\node[left] at (0.25,0.38) {$x_t$};
\node[above] at (3.45,2.10) {$x_r$};
\foreach \p in {a,c,b} \fill[black!85] (\p) circle (1.45pt);
\end{tikzpicture}\par
\vspace{1pt}{\small\bfseries (a) Source (src)\par}
\end{minipage}\hfill
\begin{minipage}[t]{0.49\linewidth}
\centering
\begin{tikzpicture}[x=1cm,y=1cm,font=\scriptsize,
  line cap=round,line join=round,>=stealth]
\path[use as bounding box] (0,-0.14) rectangle (3.85,2.78);
\coordinate (a) at (0.25,0.35);
\coordinate (c) at (2.80,1.93);
\coordinate (b) at (3.45,2.00);
\draw[black!35,line width=1.1pt]
  (a) .. controls (1.80,0.35) and (1.90,2.00) .. (b);
\draw[->,black!65,line width=1.1pt,dash pattern=on 3pt off 2pt]
  (a) .. controls (0.12,2.42) and (2.64,2.72) .. (b);
\node at (1.76,2.56) {$\psi_{t\to r-\iota}$};
\draw[->,black!65,line width=1.1pt,dash pattern=on 3pt off 2pt]
  (a) .. controls (1.92,0.02) and (2.95,0.58) .. (c);
\node at (2.00,0.16) {$\psi_{t\to r}$};
\draw[->,black,line width=1.4pt] (c)--(b);
\node at (3.13,1.60) {$-\iota v_r$};
\node[above] at (2.65,1.99) {$x_r$};
\node[left] at (0.25,0.38) {$x_t$};
\node[above] at (3.45,2.10) {$x_{r-\iota}$};
\foreach \p in {a,c,b} \fill[black!85] (\p) circle (1.45pt);
\end{tikzpicture}\par
\vspace{1pt}{\small\bfseries (b) Destination (dst)\par}
\end{minipage}
\caption{\textbf{Source and destination consistency.}
Gray S-curves show the continuous trajectory; dashed arcs show separate long- and short-range flow-map jumps, and black straight arrows show local velocity steps.
\textbf{(a) Source:} the direct map $\psi_{t\to r}$ agrees to first order with a step
$x_{t-\iota}=x_t-\iota v_t$ followed by $\psi_{t-\iota\to r}$.
\textbf{(b) Destination:} $\psi_{t\to r-\iota}$ agrees to first order with
$\psi_{t\to r}$ followed by $x_{r-\iota}=x_r-\iota v_r$.
Here $v_t=v(x_t,t)$ and $v_r=v(x_r,r)$. These two routes give the identities in Eq.~\eqref{eq:fm-consistency-relations}.}
\label{fig:tworoutes}
\end{figure}

For the calculations below, we represent the average velocity as
$v^\theta_{t\to r}(x)=(x-\psi^\theta_{t\to r}(x))/(t-r)$ for $r<t$.
Then
$J_x\psi^\theta_{t\to r}-\Id=-(t-r)J_xv^\theta_{t\to r}$, so the
Jacobian condition in \eqref{eq:fm-error-assumptions} is equivalent to
the bound on $J_xv^\theta_{t\to r}$ used below.

\begin{appendixassumption}[Average-velocity representation]
\label{ass:app-flow}
Assume Assumption~\ref{ass:fm-regularity}. At fixed $\theta$, write
$\psi^\theta_{t\to r}(x)=x-(t-r)v^\theta_{t\to r}(x)$ and identify
$v^\theta(x,t)=v^\theta_{t\to t}(x)$. The average-velocity representation
has the derivatives and continuous boundary limits used below.
\end{appendixassumption}
In this parameterization, the main-text residuals become, for $h=t-r$,
\[
 c^{\theta,\mathrm{src}}
 =v^\theta_{t\to r}+h(\partial_tv^\theta_{t\to r}+(J_x v^\theta_{t\to r})v^\theta)-v^\theta,
 \qquad
 c^{\theta,\mathrm{dst}}(x,t,r)
 =\partial_r\psi^\theta_{t\to r}(x)
       -v^\theta(\psi^\theta_{t\to r}(x),r).
\]
In the source residual, $v^\theta_{t\to r}$ and its derivatives are evaluated at
$x$ and $v^\theta$ at $(x,t)$; $r$ is fixed in $\partial_t$.
The superscript $v$ in $\psi^v$ denotes the flow map induced by the
velocity field $v$, whereas $\psi^\theta$ denotes the learned student map.

\paragraph{Source and destination identities.}
\label{app:local-consistency-derivation}
Consider an interior interval $r<t$ and sufficiently small
$\iota>0$. Differentiability of the map gives the source expansion
\begin{equation}
\begin{aligned}
 &\psi^\theta_{t-\iota\to r}
       (x-\iota v^\theta(x,t))\\
 &\quad=\psi^\theta_{t\to r}(x)
 -\iota\bigl[\partial_t\psi^\theta_{t\to r}(x)
 +(J_x\psi^\theta_{t\to r}(x))v^\theta(x,t)\bigr]
 +o(\iota).
\end{aligned}
\end{equation}
Equating this with the direct-map prediction to first order, dividing
by $\iota$, and taking $\iota\downarrow0$ yields
$\partial_t\psi^\theta+(J_x\psi^\theta)v^\theta=0$.
For the destination,
\begin{equation}
 \psi^\theta_{t\to r-\iota}(x)
 =\psi^\theta_{t\to r}(x)
  -\iota\partial_r\psi^\theta_{t\to r}(x)+o(\iota).
\end{equation}
Comparing this expansion with the short velocity step from
$\psi^\theta_{t\to r}(x)$ gives
$\partial_r\psi^\theta_{t\to r}(x)
=v^\theta(\psi^\theta_{t\to r}(x),r)$.
Thus both first-order relations give
\eqref{eq:fm-consistency-relations}; conversely, these derivative
identities imply the first-order relations by the same expansions.
Boundary statements follow by the assumed continuous limits.

\paragraph{Diagonal matching alone.}
\label{app:diagonal-example}
For a stationary teacher $v^T=0$, take
$v^\theta_{t\to r}(x)=-(t-r)a$ with a nonzero constant vector $a$.
Then $v^\theta(x,t)=v^\theta_{t\to t}(x)=0=v^T(x,t)$, but
$\psi^\theta_{t\to r}(x)=x+(t-r)^2a\ne x=\psi^T_{t\to r}(x)$.
Thus exact diagonal matching permits an incorrect finite displacement.
The source residual detects it: $c^{\theta,\mathrm{src}}=-2(t-r)a$.

We now prove the flow--velocity identification result in
Theorem~\ref{thm:fm-exact}.
\begin{proof}
For source consistency, fix $r<t$ and let
$x_q=\psi^v_{t\to q}(x)$. Along this trajectory,
\begin{equation}
 \frac{d}{dq}\psi_{q\to r}(x_q)
 =\partial_q\psi_{q\to r}(x_q)
  +(J_x\psi_{q\to r}(x_q))v(x_q,q)=0.
 \label{eq:app-source-characteristic}
\end{equation}
The identity boundary therefore gives
\begin{equation}
 \psi_{t\to r}(x)=\psi_{r\to r}(x_r)=x_r
 =\psi^v_{t\to r}(x).
 \label{eq:app-exact-source}
\end{equation}
For destination consistency, the assumed relation gives
$\partial_q\psi_{t\to q}(x)=v(\psi_{t\to q}(x),q)$
with $\psi_{t\to t}(x)=x$. Uniqueness of the ODE solution yields
$\psi_{t\to q}(x)=\psi^v_{t\to q}(x)$, proving the same conclusion.
\end{proof}

For the student pair, apply this theorem with $v=v^\theta$ and
$\psi=\psi^\theta$. If also $v^\theta=v^T$, this gives
\eqref{eq:fm-teacher-identification}; composing maps then recovers the
teacher transport. Native map and exact velocity
rollouts from the same initial noise then coincide on every valid
grid, and hence have identical joint distributions of states and comparison times under the same time-selection rule.

For the average-velocity representation, the same relations follow
from $\psi^\theta=x-hv^\theta_{t\to r}$, $h=t-r$, since
\begin{equation}
\begin{aligned}
 \partial_t\psi^\theta&=-v^\theta_{t\to r}-h\partial_tv^\theta_{t\to r},\\
 J_x\psi^\theta&=\Id-h(J_xv^\theta_{t\to r}),\\
 \partial_r\psi^\theta&=v^\theta_{t\to r}-h\partial_rv^\theta_{t\to r}.
\end{aligned}
\label{eq:app-map-derivatives}
\end{equation}
Substitution gives the MeanFlow residuals.

\begin{appendixcorollary}[Identification from zero population losses]
\label{cor:app-zero-loss}
Assume Assumption~\ref{ass:app-flow}. Let the velocity error and one of
the consistency residuals be continuous on their respective required
domains. If each has zero weighted squared population loss under a query
distribution of full support on that domain, with its weight strictly positive
almost everywhere, then both residuals vanish throughout their domains.
Consequently, $v^\theta=v^T$ and Theorem~\ref{thm:fm-exact} gives
$\psi^\theta=\psi^{v^\theta}=\psi^T$.
\end{appendixcorollary}
\begin{proof}
A zero nonnegative population loss yields a zero residual almost
everywhere under its query distribution. To conclude a pointwise identity on a
required domain, continuity of the residual and full support of that
query distribution on the domain are sufficient: a nonzero value would, by
continuity, yield a neighborhood with positive residual and positive
query mass, contradicting zero loss. Velocity and consistency queries
require their respective support conditions and positive weights.
For a fixed finite grid alone, vanishing residuals at its sampled nodes
do not establish the differential identities between nodes.

\end{proof}

\subsection{Source-consistency error bound}
\label{app:local-error}

Throughout this section $\psi^T$ denotes the flow of the teacher velocity $v^T$. Write $e^\theta=v^\theta-v^T$ and
$\epsilon_\psi(x,t,r):=\|\psi^\theta_{t\to r}(x)-\psi^T_{t\to r}(x)\|$.

\begin{appendixtheorem}[Error transfer through source consistency]
\label{thm:fm-finite-error}\label{thm:app-local-error}
Under the transport regularity conditions of
Assumption~\ref{ass:fm-regularity}, consider a source $x$ and interval
$[r,t]$, and let $x_q=\psi^T_{t\to q}(x)$ be the teacher-trajectory
state at time $q$. Suppose for $q\in(r,t]$ that
\begin{equation}
 \begin{aligned}
 \|c^{\theta,\mathrm{src}}(x_q,q,r)\|&\leq\epsilon_c,\qquad
 \|e^\theta(x_q,q)\|\leq\epsilon_v,\\
 \|J_x\psi^\theta_{q\to r}(x_q)-\Id\|_{\mathrm{op}}&\leq M(q-r).
 \end{aligned}
 \label{eq:fm-error-assumptions}
\end{equation}
Then, with $h=t-r$,
\begin{equation}
 \epsilon_\psi(x,t,r)
 \leq h\epsilon_c+(h+Mh^2/2)\epsilon_v.
 \label{eq:fm-local-error}
\end{equation}
\end{appendixtheorem}
\begin{proof}
The predictor in Eq.~\eqref{eq:indsrc} has teacher mismatch
\begin{equation}
    R^{\theta,\mathrm{src}}
      =v^\theta_{t\to r}+h\big(\partial_tv^\theta_{t\to r}+(J_x v^\theta_{t\to r})v^\theta\big)
          -v^T
      =c^{\theta,\mathrm{src}}+e^\theta.
    \label{eq:app-induced-source-residual}
\end{equation}
The residual appropriate to a \emph{teacher} characteristic instead is
\begin{equation}
\begin{aligned}
    R^{T,\theta}
      &:=v^\theta_{t\to r}+h\big(\partial_tv^\theta_{t\to r}+(J_x v^\theta_{t\to r})v^T\big)-v^T\\
      &=c^{\theta,\mathrm{src}}+(J_x\psi^\theta_{t\to r})e^\theta.
\end{aligned}
\label{eq:app-teacher-residual}
\end{equation}
All source quantities in these identities are evaluated at $(x,t,r)$,
with both velocities evaluated at $(x,t)$.

For $x_q=\psi^T_{t\to q}(x)$, the chain rule gives
\begin{equation}
    \frac{d}{dq}\psi_{q\to r}^\theta(x_q)
       =-R^{T,\theta}(x_q,q,r).
\end{equation}
Integration from $r$ to $t$, using the identity boundary, yields
\begin{equation}
    \psi_{t\to r}^\theta(x)-\psi^T_{t\to r}(x)
       =-\int_r^t R^{T,\theta}(x_q,q,r)\,dq.
    \label{eq:app-map-error-identity}
\end{equation}
Under \eqref{eq:fm-error-assumptions}, with interval length $q-r$ at
the integration point,
\begin{align}
    \|\psi_{t\to r}^\theta(x)-\psi^T_{t\to r}(x)\|
    &\leq\int_r^t
         \big[\epsilon_c+(1+M(q-r))\epsilon_v\big]\,dq \notag\\
    &=h\epsilon_c+(h+Mh^2/2)\epsilon_v.
    \label{eq:app-local-error-bound}
\end{align}
\end{proof}

\subsection{Destination-consistency error bound}
\label{app:destination-transfer}

\begin{appendixtheorem}[Error transfer through destination consistency]
\label{thm:fm-destination-error}
Under the transport regularity conditions of
Assumption~\ref{ass:fm-regularity}, suppose the teacher velocity is
$L_T$-Lipschitz on the comparison domain. If
$\|c^{\theta,\mathrm{dst}}(x,t,q)\|\leq\epsilon_c$ and
$\|e^\theta(\psi^\theta_{t\to q}(x),q)\|\leq\epsilon_v$ for $q\in[r,t]$, then
\begin{equation}
 \epsilon_\psi(x,t,r)
 \leq h e^{L_T h}(\epsilon_c+\epsilon_v),\quad h=t-r.
 \label{eq:fm-destination-error}
\end{equation}
\end{appendixtheorem}

Along the predicted arrival curve $\psi_{t\to q}^\theta(x)$,
the teacher-relative destination residual is
\begin{equation}
\begin{aligned}
    R^{\theta,\mathrm{dst}}(x,t,q)
      &=\partial_q\psi_{t\to q}^\theta(x)-v^T(\psi^\theta_{t\to q}(x),q)\\
      &=c^{\theta,\mathrm{dst}}(x,t,q)+e^\theta(\psi^\theta_{t\to q}(x),q).
\end{aligned}
\label{eq:app-destination-residual}
\end{equation}
\begin{proof}[Proof of Theorem~\ref{thm:fm-destination-error}]
\label{prop:app-destination}
Both curves start from $x$ at time $t$. Integrating their derivatives
backward and using the teacher's Lipschitz bound gives
\[
 \|\psi^\theta_{t\to r}(x)-\psi^T_{t\to r}(x)\|
 \leq\int_r^t\left[
 L_T\|\psi^\theta_{t\to q}(x)-\psi^T_{t\to q}(x)\|
 +\|R^{\theta,\mathrm{dst}}(x,t,q)\|\right]dq.
\]
Reverse-time Gronwall therefore yields
\begin{equation}
 \epsilon_\psi(x,t,r)
 \leq\int_r^t e^{L_T(q-r)}
             \|R^{\theta,\mathrm{dst}}(x,t,q)\|\,dq.
 \label{eq:app-destination-bound}
\end{equation}
Substituting $\|R^{\theta,\mathrm{dst}}\|\leq\epsilon_c+\epsilon_v$
and $e^{L_T(q-r)}\leq e^{L_T h}$ proves
\eqref{eq:fm-destination-error}.
\end{proof}

\paragraph{Uniform errors and native composition.}
Suppose these bounds hold with common
constants along every arrival curve launched from native source states.
For a grid from $1$ to $0$, define
$x_{t_{i+1}}^\psi=\psi^\theta_{t_i\to t_{i+1}}(x_{t_i}^\psi)$ and
$x_{t_i}^T=\psi^T_{1\to t_i}(x_{t_0}^\psi)$. Then
\begin{equation}
 \|x_{t_N}^{\psi}-x_{t_N}^{T}\|
 \leq e^{L_T}(\epsilon_c+\epsilon_v).
 \label{eq:fm-destination-terminal}
\end{equation}
To see this, let $D_i=\|x_{t_i}^{\psi}-x_{t_i}^T\|$ and $h_i=t_i-t_{i+1}$.
Teacher-flow stability gives
\[
 D_{i+1}\leq e^{L_T h_i}D_i
       +h_i e^{L_T h_i}(\epsilon_c+\epsilon_v).
\]
With $D_0=0$, iteration yields
$D_N\leq(\epsilon_c+\epsilon_v)\sum_i h_i e^{L_T t_i}
\leq e^{L_T}(\epsilon_c+\epsilon_v)$, since $\sum_i h_i=1$.
This proves \eqref{eq:fm-destination-terminal}, conditionally on any
valid grid for which the bounds hold.

\subsection{Error propagation through native rollouts}
\label{app:grid-propagation}

For a fixed grid, define
\begin{equation}
 x_{t_0}^{\psi}=x_{t_0}^T\sim p_1,\quad
 x_{t_{i+1}}^{\psi}=\psi^\theta_{t_i\to t_{i+1}}(x_{t_i}^{\psi}),\quad
 x_{t_{i+1}}^T=\psi^T_{t_i\to t_{i+1}}(x_{t_i}^T).
\label{eq:fm-paired-rollouts}
\end{equation}
Here $x_{t_i}^{\psi}$ and $x_{t_i}^T$ are the random states obtained
by drawing the shared initial state from $p_1$.
Let $\pi^{\psi,\Gcal,\theta}$ be the native output distribution and
$\pi^T=(\psi^T_{1\to0})_\#p_1$. Write $W_2$ for the quadratic
Wasserstein distance.

\begin{appendixtheorem}[Gridwise and distributional deployment]
\label{thm:app-deployment}
Assume Assumption~\ref{ass:app-flow}, and that $v^T$ is globally
$L_T$-Lipschitz on the comparison domain. Suppose the hypotheses of
Theorem~\ref{thm:app-local-error} hold with common constants for all
teacher segments launched from native source states, for $p_1$-almost
every initial state. Then, for $h_i=t_i-t_{i+1}$,
\begin{equation}
 \|x_{t_N}^{\psi}-x_{t_N}^T\|\leq
 \mathcal B(\Gcal):=e^{L_T}\left[\epsilon_c+
       \left(1+\frac M2\sum_i h_i^2\right)\epsilon_v\right]
 \quad\text{almost surely}.
 \label{eq:fm-terminal-error}
\end{equation}
If the output distributions have finite second moments,
$W_2(\pi^{\psi,\Gcal,\theta},\pi^T)\leq\mathcal B(\Gcal)$.
For an independent grid drawn from a deployment-grid distribution
$\gamma_{\mathrm{dep}}$, if the
hypotheses hold conditionally almost everywhere, the mixed native output
distribution has finite second moment, and $\mathbb E_{\Gcal}\mathcal B(\Gcal)^2<\infty$, then
\[
 W_2\!\left(\int\pi^{\psi,\Gcal,\theta}\gamma_{\mathrm{dep}}(d\Gcal),\pi^T\right)
 \leq\left(\mathbb E_{\Gcal}\mathcal B(\Gcal)^2\right)^{1/2}.
\]
\end{appendixtheorem}
\begin{proof}
For the coupled rollouts \eqref{eq:fm-paired-rollouts}, let
$D_i=\|x_{t_i}^{\psi}-x_{t_i}^T\|$. Add and subtract
$\psi^T_{t_i\to t_{i+1}}(x_{t_i}^{\psi})$ to obtain
\begin{align}
 D_{i+1}
 &\leq
   \|\psi_{t_i\to t_{i+1}}^\theta(x_{t_i}^{\psi})
     -\psi^T_{t_i\to t_{i+1}}(x_{t_i}^{\psi})\| \notag\\
 &\quad+
   \|\psi^T_{t_i\to t_{i+1}}(x_{t_i}^{\psi})
     -\psi^T_{t_i\to t_{i+1}}(x_{t_i}^T)\| \notag\\
 &\leq\beta_i+e^{L_Th_i}D_i,
 \label{eq:app-recursion}
\end{align}
where $\beta_i=h_i\epsilon_c+(h_i+Mh_i^2/2)\epsilon_v$.
Since $D_0=0$ and $\sum_i h_i=1$,
\begin{equation}
 D_N\leq e^{L_T}\sum_i\beta_i
     =e^{L_T}\left[\epsilon_c+
             \left(1+\tfrac M2\sum_i h_i^2\right)\epsilon_v\right]
     =:\mathcal B(\Gcal).
 \label{eq:app-terminal-bound}
\end{equation}
This proves the samplewise deployment bound.

For output distributions with finite second moments, the shared-noise
construction is an admissible coupling. Hence
\begin{equation}
    W_2(\pi^{\psi,\Gcal,\theta},\pi^T)
       \leq\big(\mathbb E[D_N^2\mid \Gcal]\big)^{1/2}
       \leq\mathcal B(\Gcal).
    \label{eq:app-output-wasserstein}
\end{equation}
For a grid sampled independently of the initial noise, the mixed native
output distribution obeys
\begin{equation}
    W_2\!\left(\int\pi^{\psi,\Gcal,\theta}\gamma_{\mathrm{dep}}(d\Gcal),\pi^T\right)
       \leq\left(\int\mathcal B(\Gcal)^2
                               \gamma_{\mathrm{dep}}(d\Gcal)\right)^{1/2}.
    \label{eq:app-random-output-wasserstein}
\end{equation}
The teacher marginal remains $\pi^T$ for every grid because it uses an
exact flow with the same initial distribution. A numerical teacher solver
introduces an additional integration error. The bound
$\sum_i h_i^2\leq1$ yields a grid-independent estimate.

\end{proof}

\begin{appendixcorollary}[Approximate state-marginal matching]
\label{cor:app-occupancy}
\label{cor:fm-occupancy}
Assume Assumption~\ref{ass:app-flow} for the student field and let it
be $L_\theta$-Lipschitz on the comparison domain. Suppose
$\|c^{\theta,\mathrm{src}}\|\leq\epsilon_c$ along all student-velocity
segments launched from native states with their associated destinations,
almost surely. Let $x_{t_i}^{v}=\psi^{v^\theta}_{1\to t_i}(x_{t_0}^{\psi})$ denote the exact student-velocity rollout. Then, under shared initial noise,
\[
 \|x_{t_i}^{\psi}-x_{t_i}^v\|\leq(1-t_i)e^{L_\theta(1-t_i)}\epsilon_c.
\]
When the two state distributions have finite second moments, the same bound holds
for their $W_2$ distance. Here $d_i^{\psi,\theta}(\cdot\mid\Gcal)$ and $d_i^{v,\theta}(\cdot\mid\Gcal)$ denote the native-map and exact-velocity rollout marginals, respectively; they need not equal the direct-map pushforward $(\psi^\theta_{1\to t_i})_\#p_1$. The statements hold conditionally on any
admissible fixed or independently sampled grid.
\end{appendixcorollary}
\begin{proof}
Use $v^T=v^\theta$ in the preceding argument, so the
velocity error is zero and each local error is bounded by
$h_i\epsilon_c$. Up to time $t_i$, the total elapsed time is
$1-t_i$, giving
\begin{equation}
    \|x_{t_i}^{\psi}-x_{t_i}^v\|
       \leq(1-t_i)e^{L_\theta(1-t_i)}\epsilon_c.
\end{equation}
The same coupling gives
\begin{equation}
    W_2\!\left(d_i^{\psi,\theta}(\cdot\mid \Gcal),
                d_i^{v,\theta}(\cdot\mid \Gcal)\right)
       \leq(1-t_i)e^{L_\theta(1-t_i)}\epsilon_c
    \label{eq:app-occupancy-wasserstein}
\end{equation}
when second moments exist. Substituting the paired-state estimate at the same $\theta$ into
\eqref{eq:rk-paired-mismatch} gives a fixed-kernel acquisition bound
with the sampled queries held fixed during optimization.

\end{proof}

\subsection{Population training losses and coverage}
\label{app:population-transfer}
\label{sec:fm-population}

Let $\Gcal\sim\gamma_{\mathrm{dep}}$ be a deployment grid independent of
the initial noise, and let $x_{t_i}^{\psi}$ be the current student's native states.
Write $x_{i,q}=\psi^T_{t_i\to q}(x_{t_i}^{\psi})$. For every bounded measurable test function $f$ on the corresponding
state--time domain, define the two analysis distributions by
\begin{align}
    \int f\,d\nu_v
       &=\mathbb E\sum_i\int_{t_{i+1}}^{t_i}
                         f(x_{i,q},q)\,dq,
       \label{eq:app-nu-v}\\
    \int f\,d\nu_c
       &=\mathbb E\sum_i\int_{t_{i+1}}^{t_i}
                         f(x_{i,q},q,t_{i+1})\,dq.
       \label{eq:app-nu-c}
\end{align}
They have total mass one because $\sum_i h_i=1$.
These analysis measures describe the transport segments used in the proof.

Let $\mu_v$ and $\mu_c$ be the actual \emph{population} distributions of training states and comparison times for the two residuals. They may be induced by frozen student
rollouts and by auxiliary queries. Define
\begin{equation}
    \mathcal L_v=\int w_v\|e^\theta\|^2\,d\mu_v,
    \qquad
    \mathcal L_c=\int\|c^{\theta,\mathrm{src}}\|^2\,d\mu_c.
    \label{eq:app-training-losses}
\end{equation}
The coverage assumption is
\begin{equation}
    \nu_v\ll\mu_v,\quad
    \frac{d\nu_v}{d\mu_v}\leq C_v\quad\mu_v\text{-a.e.},
    \qquad
    \nu_c\ll\mu_c,\quad
    \frac{d\nu_c}{d\mu_c}\leq C_c\quad\mu_c\text{-a.e.}
    \label{eq:app-coverage-ratios}
\end{equation}
These bounds imply $\int f\,d\nu_v\leq C_v\int f\,d\mu_v$ and
$\int f\,d\nu_c\leq C_c\int f\,d\mu_c$ for every nonnegative $f$.
They supply the link from deployment errors to population training losses.

\begin{appendixtheorem}[Covered population-loss transfer]
\label{thm:app-population}
\label{cor:fm-population}
Assume Assumption~\ref{ass:app-flow}; $v^T$ is $L_T$-Lipschitz on
the comparison domain; and $\|J_x v^\theta_{t\to r}\|_{\mathrm{op}}\leq M$
along the teacher segments defining $\nu_c$. Let the population training
measures $\mu_v,\mu_c$ have finite losses
\eqref{eq:app-training-losses}. Assume
\eqref{eq:app-coverage-ratios} with finite $C_v,C_c$ and
$w_v\geq w_{\min}>0$ $\mu_v$-almost everywhere. Then
\[
 \mathbb E\|x_{t_N}^{\psi}-x_{t_N}^T\|^2
 \leq 2e^{2L_T}\left[C_c\mathcal L_c+
              \frac{(1+M)^2C_v}{w_{\min}}\mathcal L_v\right].
\]
The expectation includes an independent random deployment grid when one
is used. If both output distributions have finite second moments, the same
right-hand side bounds $W_2^2(\pi^{\psi,\theta},\pi^T)$, where
$\pi^{\psi,\theta}=\int\pi^{\psi,\Gcal,\theta}\gamma_{\mathrm{dep}}(d\Gcal)$.
\end{appendixtheorem}
\begin{proof}
Let $D_N=\|x_{t_N}^{\psi}-x_{t_N}^T\|$ and set
$\phi_i(q)=\|c^{\theta,\mathrm{src}}(x_{i,q},q,t_{i+1})\|+(1+M)\|e^\theta(x_{i,q},q)\|$.
The local error identity and the grid recursion imply
\begin{equation}
    D_N\leq e^{L_T}\sum_i\int_{t_{i+1}}^{t_i}\phi_i(q)\,dq.
    \label{eq:app-integrated-terminal-error}
\end{equation}
The total integration length is one. Cauchy--Schwarz and
$(a+b)^2\leq2a^2+2b^2$ therefore give
\begin{equation}
    \mathbb E D_N^2
      \leq2e^{2L_T}
          \left[\int\|c^{\theta,\mathrm{src}}\|^2\,d\nu_c
                 +(1+M)^2\int\|e^\theta\|^2\,d\nu_v\right].
    \label{eq:app-expected-error-intermediate}
\end{equation}
If \eqref{eq:app-coverage-ratios} holds and
$w_v\geq w_{\min}>0$ on the relevant training support, then
\begin{equation}
    \int\|c^{\theta,\mathrm{src}}\|^2\,d\nu_c\leq C_c\mathcal L_c,
    \qquad
    \int\|e^\theta\|^2\,d\nu_v
       \leq\frac{C_v}{w_{\min}}\mathcal L_v.
\end{equation}
Combining these inequalities proves the expected-output-error bound.
With finite second moments,
\begin{equation}
    W_2^2(\pi^{\psi,\theta},\pi^T)
      \leq\mathbb E D_N^2
      \leq2e^{2L_T}
        \left[C_c\mathcal L_c+
        \frac{(1+M)^2C_v}{w_{\min}}\mathcal L_v\right].
    \label{eq:app-population-wasserstein}
\end{equation}
\end{proof}
The constants quantify both off-trajectory coverage and mismatch between
current deployment and a frozen acquisition distribution. A finite
empirical buffer or randomized times alone do not establish bounded
ratios. If a schedule omits endpoint intervals, the omitted transport
requires an additional error term or a separate boundary guarantee.

\section{Supervision Kernels and Model Compatibility}\label{app:kernels}
This section derives the Gaussian comparison kernels of Section~\ref{sec:fm-transfer} and lists the supervision choices each parameterization supports (Table~\ref{tab:compatibility}).
\subsection{Schedule-compatible local velocity kernels}
Let $X_0$ be a data sample and $\xi\sim\mathcal N(0,\Id)$ independent
noise. For differentiable schedule coefficients $\alpha_t,\beta_t$,
the affine path is $X_t=\alpha_tX_0+\beta_t\xi$, with density $p_t$. Its exact velocity
is $v(x,t)=\mathbb E[\dot\alpha_tX_0+\dot\beta_t\xi\mid X_t=x]$.
Writing $D_t=\alpha_t\dot\beta_t-\dot\alpha_t\beta_t$, elimination
of the conditional data prediction gives
$\alpha_tv-\dot\alpha_tx=D_t\mathbb E[\xi\mid X_t=x]$.
Since the score is $s_v(x,t)=-\mathbb E[\xi\mid X_t=x]/\beta_t$,
for $\beta_tD_t\neq0$ we obtain
\begin{equation}
 s_v(x,t)=\frac{\dot\alpha_t x-\alpha_t v(x,t)}{\beta_tD_t}.\end{equation}
For a velocity prediction $w=w(x,t)$, substitute this score into the
reverse SDE $dx_t=b_t(x_t)\,dt+g_t\,d\overline W_t$ of Section~\ref{sec:background}. An Euler step of length
$\delta>0$ and diffusion coefficient $g_t>0$ has mean and covariance
\begin{equation}
 m_w=x-\delta\left[w+\frac{g_t^2}{2\beta_tD_t}(\alpha_t w-\dot\alpha_t x)\right],
 \qquad \Sigma_t=g_t^2\delta\Id.
\end{equation}
Define $P_w^\delta(\cdot\mid x,t)=\mathcal N(m_w,\Sigma_t)$.
For student and teacher at the same $(x,t)$ with shared covariance,
the mean difference is
$-\delta(1+g_t^2\alpha_t/(2\beta_tD_t))(v^\theta-v^T)$.
The Gaussian KL formula therefore gives
\begin{equation}
 \KL(P_{v^\theta}^\delta\|P_{v^T}^\delta)
 =\frac{\delta}{2g_t^2}\left(1+\frac{g_t^2\alpha_t}{2\beta_tD_t}\right)^2
                 \norm{v^\theta-v^T}^2.
 \label{eq:generalkl}
\end{equation}
For rectified flow, $\alpha_t=1-t$, $\beta_t=t$, and $D_t=1$, giving
\begin{equation}
 m_w=x-\delta\left[w+\frac{g_t^2}{2t}\{x+(1-t)w\}\right],\qquad
 w_{\mathrm{KL}}(t,\delta)
 =\frac{\delta}{2g_t^2}\left(1+\frac{g_t^2(1-t)}{2t}\right)^2.
 \label{eq:rectifiedkl}
\end{equation}
The weight $w_v$ of Eq.~\eqref{eq:vc} is the per-unit-time weight $w_{\mathrm{KL}}(t,\delta)/\delta$ of the continuous-time form~\eqref{eq:rk-time-integral}, where $\delta$ is the kernel length and is chosen independently of the rollout step $h_i$.
For TrigFlow with data scale $\sigma_d>0$, use $\alpha_t=\cos(\pi t/2)$, $\beta_t=\sigma_d\sin(\pi t/2)$, and $D_t=(\pi/2)\sigma_d$. The Gaussian KL is exact for the displayed kernels. Marginal preservation additionally requires compatible dynamics and exact integration. Kernel noise and weighting can be chosen separately from rollout noise.

\subsection{Anchored flow-map kernels}
Instantiate the local-anchor construction of FlowMap-GRPO \citep{li2026flowmapgrpo}. Under our reverse-time convention, let $0<\varepsilon<r<t\leq1$ and $\lambda>0$.
The sampling rule of the stochastic flow-map comparison is
\begin{equation}
 \widetilde x_r^\theta
 =\psi^\theta_{t\to r}(x)-\varepsilon\lambda^2\left[\frac{\psi^\theta_{t\to r}(x)}{1-r}+v^\theta(\psi^\theta_{t\to r}(x),r)\right]
      +\lambda\sqrt{\frac{2\varepsilon r}{1-r}}\,\xi,
 \qquad\xi\sim\mathcal N(0,\Id).
 \label{eq:finite-kernel-main}
\end{equation}
The teacher uses the same rule with $\theta$ replaced by $T$.
Thus $P^\theta$ and $P^T$ have respective means $\mu^\theta,\mu^T$
and shared variance $\nu_r^2=2\varepsilon\lambda^2r/(1-r)$.
Subtracting their means and applying the Gaussian KL formula gives
\begin{equation}
 \KL(P^\theta\|P^T)
 =\frac{1-r}{4\varepsilon\lambda^2r}\norm{\mu^\theta-\mu^T}^2
 =\frac{1-r}{4\varepsilon\lambda^2r}
 \norm{(1-\varepsilon\lambda^2/(1-r))\Delta\psi
                      -\varepsilon\lambda^2\Delta v}^2,
 \label{eq:finitekl}
\end{equation}
where $\Delta\psi=\psi^\theta_{t\to r}(x)-\psi^T_{t\to r}(x)$ and $\Delta v=v^\theta(\psi^\theta_{t\to r}(x),r)-v^T(\psi^T_{t\to r}(x),r)$. Each velocity correction is evaluated at its model's predicted destination.
For exact compatible dynamics, the rectified-flow score satisfies
$s_v(x,r)=-[x+(1-r)v(x,r)]/r$. The correction is therefore a
Langevin step with drift $\lambda^2r\,s_v/(1-r)$ and diffusion
coefficient $\lambda\sqrt{2r/(1-r)}$, at fixed physical time $r$.
Its density equation preserves $p_r$ because
$-\nabla_x\cdot(p_rs_v)+\Delta_xp_r=0$.
Thus exact integration preserves the boundary marginal; the displayed
Gaussian is its Euler approximation over anchor length $\varepsilon$.

The corrected means also control map error when the teacher velocity
is $L_T$-Lipschitz and $1-\varepsilon\lambda^2/(1-r)-\varepsilon\lambda^2L_T>0$:
\begin{equation}
\begin{split}
 &\left(1-\frac{\varepsilon\lambda^2}{1-r}-\varepsilon\lambda^2L_T\right)
       \|\Delta\psi\|\\
 &\quad\leq\|\mu^\theta-\mu^T\|
 +\varepsilon\lambda^2
 \|v^\theta(\psi^\theta_{t\to r}(x),r)-v^T(\psi^\theta_{t\to r}(x),r)\|.
\end{split}
 \label{eq:anchoredident}
\end{equation}
To obtain this bound, split $\Delta v$ into the student--teacher
velocity difference at $\psi^\theta_{t\to r}(x)$ and the teacher-velocity
difference between the two predicted destinations. The latter is bounded
by $L_T\|\Delta\psi\|$. Apply the reverse triangle inequality to
$\mu^\theta-\mu^T=(1-\varepsilon\lambda^2/(1-r))\Delta\psi-\varepsilon\lambda^2\Delta v$.

For the deterministic regression of Section~\ref{sec:directmap}, the scaled limit is
$\nu_r^2\KL(P^\theta\|P^T)\to\norm{\psi^\theta_{t\to r}-\psi^T_{t\to r}}^2/2$ as $\lambda\to0$, provided the fields are finite. At zero noise, unequal point masses have infinite KL, so this regression uses the scaled limit.

\subsection{Induced source and destination supervision}
\begin{appendixproposition}[Consistency from induced teacher supervision]
\label{prop:app-induced-consistency}
Assume Assumption~\ref{ass:app-flow}, with $v^\theta_{t\to r}$ and its first
derivatives continuous up to the diagonal. On the continuous comparison
domain, suppose either
\[
 V^{\theta,\mathrm{src}}(x,t,r)=v^T(x,t)
 \quad\text{or}\quad
 V^{\theta,\mathrm{dst}}(x,t,r)
       =v^T(\psi^\theta_{t\to r}(x),r)
\]
holds for every admissible $(x,t,r)$, including limits $r\uparrow t$.
Then $v^\theta=v^T$, $\psi^\theta_{t\to r}=\psi^T_{t\to r}$, and both
$c^{\theta,\mathrm{src}}$ and $c^{\theta,\mathrm{dst}}$ vanish.
The same conclusion follows from either zero weighted squared
population loss when its query distribution has full support on this
domain and its weight is strictly positive almost everywhere.
\end{appendixproposition}
\begin{proof}
Use the predictors in Eq.~\eqref{eq:indsrc}
and write $h=t-r$. For source supervision, take $r\uparrow t$. The derivative correction
$h(\partial_tv^\theta_{t\to r}+(J_xv^\theta_{t\to r})v^\theta)$ vanishes, giving
$v^\theta(x,t)=v^\theta_{t\to t}(x)=v^T(x,t)$. The identity
$V^{\theta,\mathrm{src}}-v^T
=c^{\theta,\mathrm{src}}+(v^\theta-v^T)$ then gives
$c^{\theta,\mathrm{src}}=0$. Theorem~\ref{thm:fm-exact}
implies $\psi^\theta=\psi^{v^\theta}=\psi^T$.

For destination supervision, fix $(x,t)$. The assumed equality gives
$\partial_r\psi^\theta_{t\to r}(x)=v^T(\psi^\theta_{t\to r}(x),r)$
with $\psi^\theta_{t\to t}(x)=x$. Uniqueness of the teacher ODE
therefore yields $\psi^\theta_{t\to r}(x)=\psi^T_{t\to r}(x)$. Moreover,
$\partial_r\psi^\theta=v^\theta_{t\to r}-h\partial_rv^\theta_{t\to r}$ tends to
$v^\theta(x,t)$ as $r\uparrow t$, whereas the teacher side tends to
$v^T(x,t)$. Thus $v^\theta=v^T$ also holds.

In either case, the map is the exact flow of the common velocity.
Its source and destination flow identities give both zero consistency
residuals. For population losses, continuity and full support convert
zero weighted squared loss into pointwise matching by the argument in
Corollary~\ref{cor:app-zero-loss}, after which the preceding proof applies.
\end{proof}

For two-time MeanFlow, the induced source target $V^{\theta,\mathrm{src}}$ is defined in \eqref{eq:indsrc}. With $e^\theta=v^\theta-v^T$ as in Appendix~\ref{app:local-error},
its teacher mismatch is $c^{\theta,\mathrm{src}}+e^\theta$, so
\begin{equation}
 \norm{V^{\theta,\mathrm{src}}-v^T}^2
 \leq2\norm{c^{\theta,\mathrm{src}}}^2+2\norm{e^\theta}^2.
\end{equation}
Thus separate velocity and consistency penalties control the induced-source residual; finite-error cancellation can occur when only their sum is matched. A teacher-directed source variant uses $v^T$ in the JVP direction. Its residual is exactly $R^{T,\theta}$ from \eqref{eq:app-teacher-residual}; it is a direct Eulerian map-distillation control grounded in Flow Map Matching \citep{boffi2024fmm}. For destination supervision, compare
$\partial_r\psi^\theta_{t\to r}(x)$ with $v^T(\psi^\theta_{t\to r}(x),r)$. The comparison retains the source and interval and queries the teacher at
$\psi^\theta_{t\to r}(x)$. Appendix~\ref{app:stopped-jvp} specifies the gradient conventions.

\subsection{Fixed-endpoint consistency models}\label{app:endpoint-kernels}
A CM-style model predicts a fixed endpoint $G^\theta(x,t)=\psi^\theta_{t\to0}(x)$ with $G^\theta(x,0)=x$, obtained from its network output $f^\theta$ through the checkpoint's preconditioning,
\begin{equation}
 G^\theta(x,t)=c_{\rm skip}(t)x+c_{\rm out}(t)f^\theta(c_{\rm in}(t)x,c_{\rm noise}(t)).
 \label{eq:cmprecondition}
\end{equation}
Its native rollout predicts the endpoint and re-noises it to the next time,
\begin{equation}
 x_r=\alpha_rG^\theta(x_t,t)+\beta_r\xi,\qquad \xi\sim\mathcal N(0,\Id),\quad 0<r<t,
 \label{eq:scmrollout}
\end{equation}
where $\alpha_r,\beta_r$ follow the model's noise schedule.
A native CM model exposes $G^\theta(x,t)$ with a fixed destination, rather than a learned two-time $\psi^\theta_{t\to r}$. Its source residual with a compatible teacher velocity is
\begin{equation}
 C^{\theta,\mathrm{CM}}(x,t)
 =\partial_tG^\theta(x,t)+(J_x G^\theta(x,t))v^T(x,t).
\end{equation}
With $G^\theta=\psi^\theta_{t\to0}$, this sign convention gives $C^{\theta,\mathrm{CM}}=-c^{\mathrm{src}}$ for the pair $(G^\theta,v^T)$.
If $G^\theta(x,0)=x$ and this residual vanishes on the required domain, integration along a teacher characteristic yields $G^\theta(x,t)=\psi^T_{t\to0}(x)$. More generally, with $x_q=\psi^T_{t\to q}(x)$,
\begin{equation}
 \norm{G^\theta(x,t)-\psi^T_{t\to0}(x)}
 \leq\int_0^t\norm{C^{\theta,\mathrm{CM}}(x_q,q)}dq.
\end{equation}
For $t>0$, $v^\theta_{t\to0}(x)=(x-G^\theta(x,t))/t$ exposes the induced teacher-guided source predictor
$v^\theta_{t\to0}+t(\partial_t v^\theta_{t\to0}+(J_x v^\theta_{t\to0})v^T)=v^T-C^{\theta,\mathrm{CM}}$.
This construction supervises a fixed-endpoint CM in velocity space using
the teacher velocity as the closure field.

With the native rollout of Eq.~\eqref{eq:scmrollout}, an endpoint-anchor comparison with $\lambda>0$ and $\beta_r>0$ uses
\begin{equation}
 P^{\theta,\mathrm{CM}}(\cdot\mid x,t,r)
 =\mathcal N(\alpha_rG^\theta(x,t),\lambda^2\beta_r^2\Id),\qquad
 \KL(P^{\theta,\mathrm{CM}}\|P^{T,\mathrm{CM}})
 =\frac{\alpha_r^2}{2\lambda^2\beta_r^2}\norm{G^\theta-G^T}^2.
\end{equation}
Here $P^{T,\mathrm{CM}}$ replaces the student endpoint prediction by
$G^T$ with the same schedule and variance. Equal endpoint maps therefore
give equal kernels. The final data step uses deterministic regression.

\begin{table}[!htbp]
\centering\small
\caption{\textbf{Supervision choices supported by each native parameterization.}}
\label{tab:compatibility}
\begin{tabularx}{\textwidth}{@{}lYYY@{}}
\toprule
Native family & Source-induced & Destination-induced & Flow-map\\
\midrule
MeanFlow, two-time map & Self-induced or teacher-directed source predictor; instantaneous-velocity supervision with consistency is also available & Learned destination derivative at the student proposal & Direct flow-map regression or local-anchor Gaussian comparison\\
CM, endpoint plus re-noising & Endpoint source tangent with declared teacher-velocity closure & Unavailable natively: no learned destination-time derivative & Endpoint-anchor Gaussian comparison or endpoint regression\\
\bottomrule
\end{tabularx}
\end{table}

\section{Gradient Implementation}\label{app:implementation}
This section gives the gradient of the stop-gradient consistency loss of Eq.~\eqref{eq:stopped-consistency}.
\paragraph{Stop-gradient consistency update.}\label{app:stopped-jvp}
In the coordinates $(x,t,r)$, the source JVP is
\begin{equation}
 (v^\theta_{t\to r},D^\theta)
 =\operatorname{JVP}\bigl((x,t,r)\mapsto v^\theta_{t\to r}(x);
 (x,t,r);(v^\theta(x,t),1,0)\bigr).
\end{equation}
Here $D^\theta=\partial_t v^\theta_{t\to r}+(J_x v^\theta_{t\to r})v^\theta(x,t)$, with $h=t-r$, evaluated at the sampled state and times. The consistency loss of Section~\ref{sec:training} is $\widetilde\ell_c=\norm{v^\theta_{t\to r}+h\,\sg[D^\theta]-v^\theta(x,t)}^2$: the directional derivative is a target, as in MeanFlow training~\citep{geng2025meanflow}, while the average-velocity and instantaneous-velocity predictions are differentiated. The value of $\widetilde\ell_c$ is the squared source residual $\norm{c^{\theta,\mathrm{src}}}^2$, so the residual that Theorem~\ref{thm:main-transfer} bounds and that Table~\ref{tab:consistency-errors} reports is the trained quantity; the stop-gradient only selects the direction in which it is reduced, and one JVP per query suffices. Teacher parameters are frozen.

For T2I coordinates $(x,t,h)$ with $h=t-r$, the same source JVP
uses direction $(v^\theta,1,1)$, since differentiating at fixed $r$
also differentiates $h$; the forward residual is unchanged.
For induced-destination supervision, the teacher target is stopped,
including its input derivative.

If a required operator lacks JVP support, a one-sided fallback approximates
\begin{equation}
 \partial_t v^\theta_{t\to r}(x)+(J_x v^\theta_{t\to r}(x))b
 \approx\frac{v^\theta_{t+\vartheta\to r}(x+\vartheta b)-v^\theta_{t\to r}(x)}{\vartheta},
\end{equation}
with fixed direction $b$ and a valid signed $\vartheta$.

\end{document}